\documentclass[10pt,journal]{IEEEtran}
\usepackage{amsmath,amssymb,amsfonts,mathtools}
\usepackage{amsthm}
\usepackage{bm}
\usepackage{booktabs}
\usepackage{array}
\usepackage{tabularx}
\usepackage{cuted}
\usepackage{graphicx}
\usepackage{tikz}
\graphicspath{{../}}
\usepackage{enumitem}
\usepackage{xcolor}
\usepackage{cite}
\usepackage{microtype}
\usepackage{hyperref}
\usepackage{cleveref}
\hypersetup{
    colorlinks=true,
    linkcolor=blue,
    citecolor=blue,
    urlcolor=blue
}
\newcommand{\R}{\mathbb{R}}
\newcommand{\vect}[1]{\bm{#1}}
\newcommand{\tht}{\vect\theta}
\newcommand{\mom}{\vect p}
\newcommand{\vel}{\vect v}
\newcommand{\grad}{\nabla}
\newcommand{\T}{\mathsf T}

\newcommand{\diag}{\operatorname{diag}}

\newcommand{\Fgeo}{{\vect F_{\mathrm{geo}}}}

\newcommand{\dd}{\,\mathrm{d}}

\newtheorem{theorem}{Theorem}[section]
\newtheorem{lemma}[theorem]{Lemma}
\newtheorem{proposition}[theorem]{Proposition}
\newtheorem{corollary}[theorem]{Corollary}
\theoremstyle{definition}
\newtheorem{definition}[theorem]{Definition}
\newtheorem{assumption}[theorem]{Assumption}
\theoremstyle{remark}
\newtheorem{remark}[theorem]{Remark}

\newcounter{hamletalgorithm}

\title{
Hamiltonian Metric Learning and Energy-Based Training\\
\LARGE A Dissipative Geometric Framework for Optimization}
\author{Sparsho~Chakraborty\textsuperscript{1}, Mohammad~Alamgir\textsuperscript{2}, Nishanth~M\textsuperscript{1},
Akshit~Nanda\textsuperscript{3}, Ram~Prasad~Padhy\textsuperscript{3}, and~Sayan~Mukherjee\textsuperscript{4}\\[1em]
\small\textsuperscript{1}Department of Physics, School of Basic Sciences, Indian Institute of Technology Bhubaneswar, Odisha, India\\
\small\textsuperscript{2}Department of Electronics and Computer Science, Narula Institute of Technology, West Bengal, India \\
\small\textsuperscript{3}Department of Computer Science and Engineering, School of Electrical and Computer Sciences, Indian Institute of Technology Bhubaneswar, Odisha, India \\
\small\textsuperscript{4}Department of Physics, The University of Tokyo, Hongo, Bunkyo-ku, Tokyo, Japan}
\begin{document}
\maketitle

\begin{abstract}
Optimization in machine learning is usually expressed through iterative
rules that update model parameters using information from the loss
landscape. In this work, we study an alternative viewpoint in which parameter
optimization is treated as the evolution of a dissipative dynamical system.
The model parameters are regarded as generalized coordinates, the loss acts
as a potential energy, and a positive-definite metric defines the local
kinetic geometry of parameter space. Starting from a variational formulation, we derive the corresponding
Hamiltonian dynamics, the geometric force generated by a position-dependent metric, and a metric-compatible Rayleigh dissipation law.
The resulting continuous system satisfies a monotonic energy-dissipation
relation, while its discrete form allows the influence of curvature on the
optimization trajectory to be studied directly.
We illustrate the framework using a controlled CIFAR-10 image-reconstruction
problem for which the optimum is known analytically. With the image Hessian
used as the metric, the curvature dependence of the quadratic modal dynamics
is removed. We then reparameterize the same image as a matrix product state,
producing a genuinely position-dependent metric and a nonzero geometric
force. Poincar\'e return maps provide a complementary phase-space view of
the resulting contraction dynamics.
These examples establish HAMLET as a geometric, energy-based framework for
studying optimization as dissipative motion in parameter space. On a
five-seed MNIST MLP benchmark, HAMLET attains the highest mean test accuracy
(97.95\%) and the lowest mean test negative log-likelihood (0.0696) among the
three evaluated optimizers.
\\
    
\end{abstract}
\section{Introduction}
Optimization is the process of finding values of a set of variables that make an objective function as small or as large as possible, subject to any constraints of the problem.
In machine learning, the variables being optimized are usually the model parameters \(\theta\), and the objective is a loss function \(L(\theta)\). The objective information tells us how well the current parameters perform. The parameter-space information describes how changes in the parameters should be interpreted. In the simplest case, parameter space is treated as Euclidean. More sophisticated optimization methods may incorporate curvature, a metric, statistical information, or some other local structure. Finally, an update rule determines how the parameters evolve using this information. Unlike conventional optimization, which specifies an objective and a rule for moving through parameter space, we ask whether this parameter evolution can instead be formulated as the trajectory of a dynamical system whose state obeys physically motivated equations of motion. If the parameters are regarded as dynamical variables, then optimization can be viewed as motion through parameter space. In this picture, the loss landscape plays a role analogous to a potential-energy landscape, while the parameter trajectory determines how the system moves across it. A purely conservative system would generally oscillate around low-energy configurations rather than settle at them, so dissipation becomes essential if the dynamics are to converge toward a stationary state. This motivates a formulation in which the model parameters \(\theta\) are treated as generalized coordinates, their rates of change define velocities, and a parameter-space metric \(g(\theta)\) determines the local kinetic structure of the system. The trajectory followed by such a system depends not only on the loss landscape, but also on the geometry assigned to the parameter space. If the metric is Euclidean, all directions are treated uniformly. In contrast, a position-dependent positive-definite metric \(g(\theta)\) allows different directions in parameter space to carry different effective kinetic costs. Directions associated with large metric values behave as if they possess greater inertia, while directions associated with smaller metric values permit faster motion. In this way, the metric modifies how momentum is converted into parameter velocity and therefore reshapes the optimization trajectory without changing the location of the minima of the loss itself. Because the metric may vary with \(\theta\), this geometric structure also introduces additional forces into the dynamics. These arise naturally from the position dependence of the kinetic energy rather than being inserted as heuristic correction terms. The resulting motion is therefore governed jointly by the gradient of the loss, the local geometry of the parameter manifold, and dissipative effects that remove energy from the system.
This leads to the central question of the proposed framework: can such dissipative geometric dynamics provide a useful and mathematically consistent framework for machine-learning optimization? To study this, we derive the  Hamiltonian Metric based Learning and Energy based Training (HAMLET) dynamics from a variational formulation, construct the corresponding Hamiltonian system, introduce metric-compatible dissipation, evaluate the resulting optimizer on MNIST classification, and examine its optimization behavior on a controlled image-reconstruction problem. The image problem is first studied directly in pixel space, where the
quadratic objective provides a constant curvature metric and allows the
spectral action of HAMLET to be examined analytically. We then represent the
same image as a matrix product state, transforming the problem into a
nonlinear parameter manifold with a position-dependent metric. This second
representation activates the geometric-force term and provides a controlled
example of the full HAMLET dynamics. To further characterize the resulting system dynamics, we employ Poincar\'e return maps which provide a geometric description of how different curvature modes contract under dynamics and allow the effect of metric normalization to be visualized directly.

The remainder of the paper is organized into five parts.
The first develops the variational and geometric formulation of HAMLET.
The second introduces dissipation and examines the associated energy and
symmetry structure. The third derives the discrete dynamics and studies
their stability. The fourth reports an MNIST MLP benchmark. The fifth
applies the framework to a controlled image-reconstruction problem, first in pixel space and then on an MPS
parameter manifold, allowing both constant- and position-dependent metrics
to be examined explicitly, and analyzes the resulting phase-space behavior
through Poincar\'e return maps.
\section{Related Works}\label{Related works}
A substantial body of work has interpreted optimization algorithms through
continuous-time dynamical systems. One of the most influential examples is
the differential-equation interpretation of Nesterov's accelerated gradient
method developed by Su, Boyd, and Cand\`es~\cite{su-boyd-candes}. In this
view, acceleration emerges from a second-order ordinary differential
equation whose inertial and damping terms help explain the behavior of the
underlying discrete algorithm. Such formulations established a useful bridge
between iterative optimization and classical dynamical-systems analysis,
including questions of stability, asymptotic behavior, and convergence rates.

A related line of work derives optimization dynamics from variational
principles. The Bregman Lagrangian framework of Wibisono, Wilson, and
Jordan~\cite{wibisono} shows that broad families of accelerated optimization
methods can be understood through appropriately chosen Lagrangians and
time-dependent variational structures. This viewpoint is particularly
important for our work because it demonstrates that optimization
updates need not be introduced only as algebraic rules: they may instead
arise from an underlying mechanical principle. HAMLET follows this general
philosophy, but adopts a metric kinetic energy on the parameter space and carries
the corresponding geometric terms explicitly into the equations of motion.

Hamiltonian formulations provide another natural phase-space description of
optimization. Hamiltonian descent methods~\cite{maddison} introduce momentum
variables and study optimization through Hamiltonian dynamics, thereby
separating parameter coordinates from their conjugate momenta. This
phase-space perspective is closely related to classical mechanics and offers
a natural language for inertial optimization. However, conservative
Hamiltonian dynamics alone does not generically converge to a minimizer,
since trajectories may continue to oscillate around low-energy regions.
Dissipation therefore plays a central role whenever Hamiltonian dynamics are
used as a model for optimization.

Several works have addressed this issue by introducing nonconservative
geometric structures. Conformal symplectic optimization
methods~\cite{franca-conformal} study dissipative modifications of
Hamiltonian dynamics while preserving a controlled deformation of the
underlying symplectic structure. Contact-geometric approaches
\cite{bravetti-contact} provide another intrinsic description of dissipative
systems and have been used to formulate optimization dynamics in an extended
geometric setting. These works demonstrate that dissipation can be
incorporated consistently into mechanics-inspired optimization rather than
being treated merely as an ad hoc numerical damping term.

In parallel, geometric optimization methods emphasize that the parameter
space itself need not be treated as Euclidean. Natural-gradient methods
\cite{amari,martens-natural} replace the ordinary Euclidean gradient by a
geometry-aware direction, typically defined through a positive-definite
information matrix. More broadly, Riemannian optimization
\cite{absil} studies optimization over curved spaces by introducing a metric,
geodesics, tangent spaces, and geometry-consistent descent directions. These
approaches establish that the local geometry of parameter space can strongly
affect the effective direction and scale of parameter updates.

Our framework, HAMLET, is positioned at the intersection of these viewpoints. The framework
does not introduce continuous-time optimization, Hamiltonian dynamics,
dissipation, or Riemannian geometry individually as new concepts. Instead,
it combines them within a single mechanical construction in which the loss
function plays the role of potential energy and a positive-definite metric
field $g(\tht)$ determines the kinetic geometry of parameter space. The
metric defines the canonical relation
\[
\mom = g(\tht)\dot{\tht},
\]
and therefore controls how momentum is converted into parameter velocity.
When the metric depends on position, its variation contributes an additional
geometric force to the canonical momentum equation. This term is not inserted
as an external correction but follows directly from differentiating the
metric-dependent kinetic energy.

Dissipation is introduced through a Rayleigh dissipation function that is
compatible with the same parameter-space metric. As a result, the
Hamiltonian, geometric force, and damping law are derived within one
consistent variational and mechanical framework. This yields the parameter
update equation in a dissipative form with an associated continuous-time
energy law. The role of the metric is therefore twofold: it acts locally as a
direction-dependent inertial structure and, when position dependent, it also
modifies the trajectory through the induced geometric force.
\section{Mathematical Formulation}\label{Mathematical Formulation}
\subsection{From loss to Minimization to Metric Motion}\label{From loss to Minimization to Metric Motion}
\subsubsection{From a Scalar loss to Gradient}
\label{From a Scalar loss to Gradient}
Let $\tht\in\Omega\subseteq\R^n$, $n\geq1$, collect the trainable parameters. The
optimization problem is
\begin{equation}
 \min_{\tht\in\Omega}L(\tht),\qquad
 L(\tht)=\frac{1}{N}\sum_{a=1}^N L_a(\tht).
 \label{eq:objective}
\end{equation}
Define the gradient and Hessian by
\begin{equation}
 [\grad L]_i=\partial_iL,\qquad [B_L]_{ij}=\partial_i\partial_jL.
\end{equation}
Taylor expansion gives
\begin{equation}
 \begin{split}
 L(\tht+\vect d)={}&L(\tht)+\grad L^{\T}\vect d\\
 &+\tfrac12\vect d^{\T} B_L\vect d+o(\|\vect d\|^2).
 \end{split}
 \label{eq:taylor}
\end{equation}
Choosing $\vect d=-a\grad L$, $a>0$, makes the first-order change
$-a\|\grad L\|^2$. This motivates gradient descent; it does not yet introduce
inertia or a metric.

\subsubsection{Mechanical Motion}
Introduce a differentiable trajectory $\tht(t)$, velocity
$\vel=\dot\tht$, and acceleration $\dot\vel=\ddot\tht$. Here $t$ is
optimization time. The mechanical analogy identifies position with parameters
and potential with loss. For one coordinate and constant mass $m>0$, it gives
\begin{equation}
 T=\tfrac12m\dot\theta^2,\quad V=L(\theta),\quad m\ddot\theta=-L'(\theta).
 \label{eq:newton}
\end{equation}
The last equation will follow from the variational construction rather than
being imposed independently.
\subsubsection{The Metric as Direction Dependent Mass} \label{The Metric as a Direction Dependent Mass}
A symmetric positive-definite matrix $g(\tht)$ generalizes the scalar mass:
\begin{equation}
 g=g^{\T},\qquad \vect x^{\T}g\vect x>0\quad(\vect x\ne0).
 \label{eq:spd}
\end{equation}
It defines both local squared length and kinetic energy,
\begin{equation}
 \dd s^2=\dd\tht^{\T} g\dd\tht,\qquad
 T(\tht,\vel)=\tfrac12\vel^{\T} g(\tht)\vel.
 \label{eq:kinetic}
\end{equation}
An eigendirection with a large metric eigenvalue carries a greater kinetic cost
at the same velocity. This is a choice of optimization geometry, not a claim 
that trainable parameters have physical mass.
Unless stated otherwise, $L$ and $g$ are sufficiently smooth on an open domain,
$g$ is positive definite, and derivatives are taken in a fixed parameter chart.
Energy identities need differentiability; existence, uniqueness, and limit-set
arguments require the corresponding local Lipschitz and compactness conditions.
If $g$ is built from a Hessian, differentiating it requires third derivatives of
$L$; stronger smoothness of $L$ is required whenever a later result invokes the
standing assumption $g\in C^2$. Nonsmooth neural-network losses require a smooth
surrogate or a separately specified generalized-derivative treatment. For
example, ReLU activations can be replaced by the softplus
$\tau^{-1}\log(1+e^{\tau x})$, which tends to ReLU as $\tau\to\infty$, an
$L_1$ residual by the pseudo-Huber loss
$\sigma\big(\sqrt{1+(r/\sigma)^2}-1\big)$, and a hard maximum by a
log-sum-exp; each surrogate is $C^\infty$ and recovers the original loss in the
limit of its smoothing parameter.
\begin{table}[t]
\caption{Notation Used in the Derivation}
\label{tab:notation}
\centering\footnotesize
\begin{tabular}{@{}p{0.23\columnwidth}p{0.69\columnwidth}@{}}
\toprule
Symbol & Meaning\\
\midrule
$\tht,\vel,\mom$ & Parameters, velocity, canonical momentum\\
$L,B_L$ & Loss and its Hessian\\
$g,g^{-1}$ & Kinetic metric and inverse metric\\
$\mathcal L,\mathcal A$ & Lagrangian and action\\
$H_0$ & HAMLET Hamiltonian\\
$\mathcal R,\gamma$ & Rayleigh function and damping rate\\
$\Fgeo$ & Geometric force from the position dependence of $g$\\
$h,\beta$ & Discrete time step and momentum retention, $\beta=1-\gamma h$\\
$J_\xi$ & Noether Current\\
$\xi$ & Killing field of g \\
\bottomrule
\end{tabular}
\end{table}

\subsubsection{Assumptions}\label{Assumpations}
\begin{assumption}
[Autonomous core]\label{ass:core}
The domain $\Omega\subseteq\R^n$ is open, $L\in C^2(\Omega)$, and
$g\in C^2(\Omega;\mathbb S_{++}^n)$ is time independent. The damping rate
$\gamma$ is a positive constant. These assumptions give a locally Lipschitz
phase-space vector field. Compactness and uniform metric bounds are imposed
explicitly where needed; they do not follow from local smoothness alone.
\end{assumption}
\begin{remark}[Scope of the standing assumption]\label{rem:scope}
The curvature-shifted metric \eqref{eq:regularization} used later is Lipschitz
but not $C^1$ at an eigenvalue crossing of $C$ or at the switching point of
the maximum, so it satisfies Assumption~\ref{ass:core} only on the open set
where $\lambda_{\min}(C)$ is simple and the shift does not switch. Every
result stated under Assumption~\ref{ass:core} is therefore to be read on such
a smooth branch. At points where the metric remains differentiable, the energy
identity continues to hold. Extending the invariance argument across nonsmooth
switching sets would require a generalized-gradient treatment, which is outside
the present scope. A fixed shift valid over the region of interest, or a smooth
SPD construction, removes the issue entirely.
\end{remark}
\subsection{Variational Principle to Update Equations}\label{Variational Principal to Update Equat}
\subsubsection{Lagrangian and Euler--Lagrange Equation}
% ------------------------------------------------------------
Define the Lagrangian and action by
\begin{align}
 \mathcal L(\tht,\vel)&=T(\tht,\vel)-L(\tht),\\
 \mathcal A[\tht]&=\int_{t_0}^{t_1}\mathcal L(\tht,\dot\tht)\dd t.
\end{align}
The loss $L$ and Lagrangian $\mathcal L$ are different quantities. The
stationary-action assumption is $\delta\mathcal A=0$ for every perturbation
$\tht_\epsilon=\tht+\epsilon\vect\xi$ with
$\vect\xi(t_0)=\vect\xi(t_1)=0$. The chain rule gives
\begin{equation}
 \delta\mathcal A=\int_{t_0}^{t_1}\sum_i
 \left(\frac{\partial\mathcal L}{\partial\theta_i}\xi_i
 +\frac{\partial\mathcal L}{\partial v_i}\dot\xi_i\right)\dd t.
\end{equation}
Integration by parts yields
\begin{equation}
 \begin{split}
 \int_{t_0}^{t_1}\frac{\partial\mathcal L}{\partial v_i}\dot\xi_i\dd t
 ={}&\left[\frac{\partial\mathcal L}{\partial v_i}\xi_i\right]_{t_0}^{t_1}\\
 &-\int_{t_0}^{t_1}\frac{\mathrm d}{\mathrm dt}
 \left(\frac{\partial\mathcal L}{\partial v_i}\right)\xi_i\dd t .
 \end{split}
\end{equation}
The boundary term vanishes. Because the perturbations are arbitrary in the
interior, the fundamental lemma of the calculus of variations gives
\begin{equation}
 \frac{\mathrm d}{\mathrm dt}
 \left(\frac{\partial\mathcal L}{\partial v_i}\right)
 -\frac{\partial\mathcal L}{\partial\theta_i}=0.
 \label{eq:EL}
\end{equation}
For the scalar constant-mass model,
$\partial_{\dot\theta}\mathcal L=m\dot\theta$ and
$\partial_\theta\mathcal L=-L'(\theta)$, so \eqref{eq:EL} gives
\eqref{eq:newton}. Stationarity is not necessarily minimization of the action.
% ------------------------------------------------------------
\subsubsection{Canonical Momentum}
% ------------------------------------------------------------
Define $p_i=\partial\mathcal L/\partial v_i$. Differentiating the quadratic
kinetic term explicitly,
\begin{align}
 p_i&=\frac{\partial}{\partial v_i}
       \left(\tfrac12\sum_{j,k}v_jg_{jk}v_k\right)\\
 &=\tfrac12\sum_k g_{ik}v_k+\tfrac12\sum_jv_jg_{ji}
 =\sum_jg_{ij}v_j.
\end{align}
Symmetry and positive definiteness therefore imply
\begin{equation}
 \mom=g\vel,\qquad \vel=g^{-1}\mom.
 \label{eq:momentum}
\end{equation}
The positivity assumption makes the velocity-to-momentum map invertible and
the quadratic Legendre transform regular.
% ------------------------------------------------------------
\subsubsection{Legendre Transform and Hamilton's Equations}
% ------------------------------------------------------------
Define the Hamiltonian at fixed $\tht$ by
\begin{equation}
 H_0(\tht,\mom)=\mom^{\T}\vel-\mathcal L(\tht,\vel),\qquad
 \vel=g^{-1}\mom.
\end{equation}
Substitution gives
\begin{align}
 H_0&=\mom^{\T} g^{-1}\mom
 -\left[\tfrac12(g^{-1}\mom)^{\T} g(g^{-1}\mom)-L\right]\\
 &=\tfrac12\mom^{\T} g^{-1}(\tht)\mom+L(\tht).
 \label{eq:H0}
\end{align}
To derive Hamilton's equations, differentiate the defining transform:
\begin{equation}
 \begin{split}
 \dd H_0={}&\vel^{\T}\dd\mom+\mom^{\T}\dd\vel\\
 &-(\partial_{\tht}\mathcal L)^{\T}\dd\tht
 -(\partial_{\vel}\mathcal L)^{\T}\dd\vel\\
 ={}&\vel^{\T}\dd\mom-(\partial_{\tht}\mathcal L)^{\T}\dd\tht.
 \end{split}
\end{equation}
The $\dd\vel$ terms cancel by the definition of $\mom$.
Equation~\eqref{eq:EL} gives $\partial_{\tht}\mathcal L=\dot\mom$.
Matching coefficients of the independent differentials
$\dd\tht$ and $\dd\mom$ proves
\begin{equation}
 \dot\tht=\partial_{\mom} H_0,\qquad
 \dot\mom=-\partial_{\tht} H_0.
 \label{eq:hamilton}
\end{equation}
For \eqref{eq:H0}, these equations become
\begin{align}
 \dot\tht&=g^{-1}\mom,\\
 \dot p_i&=-\partial_iL-\tfrac12\mom^{\T}(\partial_i g^{-1})\mom.
 \label{eq:canonical-force}
\end{align}
In the position derivative, canonical momentum is held fixed. Differentiating
$\mom=g\vel$ inside this partial derivative while also holding $\vel$ fixed
would mix the Lagrangian and Hamiltonian coordinate systems.
\begin{definition}[Metric Hamiltonian]
The function $H_0$ in \eqref{eq:H0}, with $\mom=g\vel$, is the
HAMLET Hamiltonian. Dissipation changes its evolution law without
changing this Legendre-transform definition.
\end{definition}
% ------------------------------------------------------------
\subsubsection{Symmetries and Noether Currents}
\label{sec:noether}
% ------------------------------------------------------------
The variational form also shows which quantities the conservative HAMLET
flow preserves. Let $\xi:\Omega\to\R^n$ be a smooth vector field and
consider the infinitesimal change of coordinates
$\tht\mapsto\tht+\epsilon\,\xi(\tht)$. Velocities change by
$\vel\mapsto\vel+\epsilon\,(D\xi)\vel$, with $(D\xi)_{kj}=\partial_j\xi_k$.
The first variation of the Lagrangian is
\begin{align}
\delta_\xi\mathcal L
&=\frac{\partial\mathcal L}{\partial\tht}^{\T}\xi
 +\frac{\partial\mathcal L}{\partial\vel}^{\T}(D\xi)\vel
 =\tfrac12\vel^{\T}(\mathfrak L_\xi g)\vel-\xi^{\T}\grad L,
\label{eq:noether-variation}\\
(\mathfrak L_\xi g)_{ij}&=\sum_k\big(\xi_k\partial_kg_{ij}
 +g_{kj}\partial_i\xi_k+g_{ik}\partial_j\xi_k\big),
\label{eq:lie-derivative}
\end{align}
where $\mathfrak L_\xi g$ is the Lie derivative of the metric. The field
$\xi$ is a \emph{Killing field} of $g$ if $\mathfrak L_\xi g=0$, so that the
coordinate change is an isometry, and a \emph{symmetry of the loss} if
$\xi^{\T}\grad L\equiv0$.
\begin{definition}[Noether current]
The Noether current of $\xi$ is
\begin{equation}
J_\xi(\tht,\vel)=\frac{\partial\mathcal L}{\partial\vel}^{\T}\xi
=\mom^{\T}\xi(\tht)=\vel^{\T}g(\tht)\,\xi(\tht).
\label{eq:noether-current}
\end{equation}
\end{definition}
\begin{theorem}[Noether identity]\label{thm:noether}
Along every solution of the Euler--Lagrange equation \eqref{eq:EL},
\begin{equation}
\dot J_\xi=\delta_\xi\mathcal L
=\tfrac12\vel^{\T}(\mathfrak L_\xi g)\vel-\xi^{\T}\grad L .
\label{eq:noether-identity}
\end{equation}
In particular, if $\xi$ is a Killing field of $g$ and a symmetry of the loss,
$J_\xi$ is conserved.
\end{theorem}
\begin{proof}
By the product rule,
$\dot J_\xi=\dot{\mom}^{\T}\xi+\mom^{\T}(D\xi)\vel$. The Euler--Lagrange
equation gives $\dot{\mom}=\partial\mathcal L/\partial\tht$, so
$\dot J_\xi=\delta_\xi\mathcal L$ by the first equality in
\eqref{eq:noether-variation}. For the second, use
$\partial_{\tht}\mathcal L=\tfrac12\vel^{\T}(\partial g)\vel-\grad L$ and
$\mom^{\T}(D\xi)\vel=\tfrac12\vel^{\T}(gD\xi+D\xi^{\T}g)\vel$.
\end{proof}
Time translation is the other classical symmetry. Because $\mathcal L$ has
no explicit time dependence, its Noether charge, the Legendre transform
$H_0=\mom^{\T}\vel-\mathcal L$, is conserved by the conservative flow. With
dissipation this becomes the energy theorem of
Section~\ref{Dissipative dynamics and Energy evolution}.
In Hamiltonian form $J_\xi=\mom^{\T}\xi$ generates the lifted coordinate
change, and conservation is $\{J_\xi,H_0\}=0$.

Three cases occur in optimization. (i) If $g=I$ and the loss is flat along
a fixed direction $u$, then $J=\mom^{\T}u$ is conserved: momentum along a
flat direction never decays in the conservative flow. (ii) If $g=I$ and
$\xi=A\tht$ with $A=-A^{\T}$ leaves the loss invariant, the angular
momentum $\mom^{\T}A\tht$ is conserved. (iii) For the MPS of
Section~\ref{sec:mps}, the gauge field
$\xi_X:(A^{(k)},A^{(k+1)})\mapsto(A^{(k)}X,-XA^{(k+1)})$ leaves the image
unchanged, so $\xi_X^{\T}\grad L\equiv0$ for every $X$. It is a Killing
field of the Euclidean metric exactly when $X=-X^{\T}$. Loss symmetries that
are not isometries, such as the general linear gauge or the layer rescalings
of neural networks, are not conserved; \eqref{eq:noether-identity} gives the
exact rate at which the metric breaks them, the kinetic symmetry breaking
of Noether's learning dynamics~\cite{tanaka-nld}. This is the inertial,
metric-dependent counterpart of the broken conservation laws of gradient
flow~\cite{kunin-neural-mechanics}.
\subsection{The Geometric Force and Curved Motion}

\subsubsection{Derivative of the Inverse Metric}
% ------------------------------------------------------------
\begin{lemma}[Inverse metric and canonical geometric force]
For a differentiable SPD metric, at fixed momentum,
\[
 \partial_i g^{-1}=-g^{-1}(\partial_i g)g^{-1},
\]
\[
 F_{\mathrm{geo},i}=\tfrac12\mom^{\T}(\partial_i g^{-1})\mom
 =-\tfrac12\vel^{\T}(\partial_i g)\vel.
\]
\end{lemma}
\begin{proof}
Differentiating $g^{-1}g=I$ gives
\begin{equation}
 (\partial_i g^{-1})g+g^{-1}(\partial_i g)=0.
\end{equation}
Multiplication on the right by $g^{-1}$ proves
\begin{equation}
 \partial_i g^{-1}=-g^{-1}(\partial_i g)g^{-1}.
 \label{eq:inverse-derivative}
\end{equation}
Define the geometric force by the sign convention
\begin{align}
 F_{\mathrm{geo},i}
 &:=\tfrac12\mom^{\T}(\partial_i g^{-1})\mom\\
 &=-\tfrac12\vel^{\T}(\partial_i g)\vel.
 \label{eq:geoforce}
\end{align}
The actual geometric contribution to $\dot\mom$ is $-\Fgeo$.
Thus \eqref{eq:canonical-force} is $\dot\mom=-\grad L-\Fgeo$. For a constant metric, all
components of $\Fgeo$ vanish.
\end{proof}
\begin{remark}[Sign convention for $\Fgeo$]\label{rem:fgeo-sign}
The sign is chosen so that $\Fgeo$ plays the same role for the kinetic energy
that $\grad L$ plays for the potential: at fixed momentum,
$F_{\mathrm{geo},i}=\partial_i\big(\tfrac12\mom^{\T}g^{-1}\mom\big)$, so
$\dot\mom=-\partial_{\tht}H_0=-\grad L-\Fgeo$, and both enter as gradients to
be subtracted. The physical force exerted by the metric on $\dot\mom$ is
therefore $-F_{\mathrm{geo},i}=+\tfrac12\vel^{\T}(\partial_i g)\vel$. This is
the positive term expected from the Lagrangian picture, where
$\dot\mom=\partial_{\tht}\mathcal L=\tfrac12\vel^{\T}(\partial g)\vel-\grad L$.
The two expressions differ in sign only because $\partial_{\tht}$ acts at fixed
$\mom$ in the Hamiltonian picture and at fixed $\vel$ in the Lagrangian one:
$\partial_i T|_{\mom}=-\partial_i T|_{\vel}$ by \eqref{eq:inverse-derivative}.
The scalar check of Section~\ref{sec:scalar-check} confirms the sign.
\end{remark}
\subsubsection{Connection to the Christoffel Terms}
% ------------------------------------------------------------
Differentiate $p_i=\sum_jg_{ij}v_j$ along the trajectory:
\begin{equation}
 \dot p_i=\sum_jg_{ij}\dot v_j
 +\sum_{j,k}(\partial_k g_{ij})v_kv_j.
\end{equation}
Equating this with \eqref{eq:canonical-force} and using
\eqref{eq:geoforce} gives
\begin{equation}
 \sum_jg_{ij}\ddot\theta_j+
 \sum_{j,k}\Gamma_{ijk}v_jv_k+\partial_iL=0,
 \label{eq:lower-christoffel}
\end{equation}
where the symmetry of $v_jv_k$ permits the definition
\begin{equation}
 \Gamma_{ijk}=\tfrac12
 (\partial_jg_{ik}+\partial_kg_{ij}-\partial_ig_{jk}).
\end{equation}
Writing $g^{ai}=[g^{-1}]_{ai}$ and
$\Gamma^a{}_{jk}=\sum_i g^{ai}\Gamma_{ijk}$, multiplication by $g^{-1}$ yields
\begin{equation}
 \ddot\theta_a+\sum_{j,k}\Gamma^a{}_{jk}\dot\theta_j\dot\theta_k
 =-\sum_i g^{ai}\partial_iL.
 \label{eq:geodesic}
\end{equation}
Thus the canonical geometric term generates the connection correction after
converting momentum dynamics to acceleration. If $L$ is constant,
\eqref{eq:geodesic} is the geodesic equation; numerically integrating such
equations while conserving the underlying Hamiltonian, rather than
approximating them by discretized length minimization, is addressed
directly by~\cite{geodesics}.
% ------------------------------------------------------------
\subsubsection{Coordinate Dependence}
% ------------------------------------------------------------
For an orthogonal change $\tht=Q\vect z$, $Q^{\T} Q=I$, the transformed metric
is $g_z=Q^{\T} g_\theta Q$, and its eigenvalues are unchanged. The kinetic energy is therefore rotation
invariant.
Under a general coordinate change with Jacobian $J$, a metric transforms by
$g_z=J^{\T} g_\theta J$, which generally changes its eigenvalues. Moreover, an
ordinary Hessian has an additional coordinate-transformation term away from
stationary points, and stored gradients at different points require a transport
rule on a general manifold. Accordingly, the Hessian and metric
constructions here are specified in a fixed parameter chart; arbitrary nonlinear
reparameterization invariance is not asserted.
% ------------------------------------------------------------s
\subsubsection{A Complete Scalar Check}\label{sec:scalar-check}
% ------------------------------------------------------------
Take $L(\theta)=\theta^2/2$ and $g(\theta)=1+\theta^2$. Then
\begin{align}
 H_0&=\frac{p^2}{2(1+\theta^2)}+\frac{\theta^2}{2},\\
 \dot\theta&=\frac{p}{1+\theta^2},\qquad
 F_{\mathrm{geo}}=-\frac{\theta p^2}{(1+\theta^2)^2},\\
 \dot p&=-\theta+\frac{\theta p^2}{(1+\theta^2)^2}.
 \label{eq:scalar-check}
\end{align}
Since $p=(1+\theta^2)\dot\theta$, the same dynamics give
\begin{equation}
 (1+\theta^2)\ddot\theta+\theta\dot\theta^2+\theta=0.
\end{equation}
This verifies the inverse-metric derivative, force sign, and Christoffel form.
%=====================================================
\section{Dissipative dynamics and Energy evolution}\label{Dissipative dynamics and Energy evolution}
\subsection{Rayleigh Dissipation and Energy theorem}
\label{sec:energy-section}
\subsubsection{Why Conservative Motion Is Insufficient}
% ------------------------------------------------------------
For constant mass and $L(\theta)=a\theta^2/2$, $a>0$, the conservative equation
is $m\ddot\theta+a\theta=0$. Nontrivial solutions oscillate about the minimum
without losing energy. This motivates damping: optimization need not require
the loss itself to decrease at every instant, but persistent conservative
oscillations generally prevent settling.
In one dimension, a viscous force $-c\dot\theta$ is generated by
\begin{equation}
 \mathcal R=\tfrac12c\dot\theta^2,\qquad
 Q_{\mathrm{diss}}=-\partial_{\dot\theta}\mathcal R.
\end{equation}
The Rayleigh function describes dissipation through its velocity derivative.
\subsubsection{Metric-Compatible Dissipation}
% ------------------------------------------------------------
For the metric model choose
\begin{equation}
 \mathcal R(\tht,\vel)=\frac{\gamma}{2}\vel^{\T} g(\tht)\vel,
 \qquad\gamma>0.
 \label{eq:rayleigh}
\end{equation}
The force is
\begin{equation}
 \vect Q_{\mathrm{diss}}=-\partial_{\vel}\mathcal R
 =-\gamma g\vel=-\gamma\mom.
 \label{eq:damping-force}
\end{equation}
Adding its virtual work to stationary action gives the forced
Euler--Lagrange equation
\begin{equation}
 \frac{\mathrm d}{\mathrm dt}\partial_{\vel}\mathcal L
 -\partial_{\tht}\mathcal L=-\partial_{\vel}\mathcal R.
\end{equation}
The Legendre transform still defines $H_0$, so
\begin{equation}
 \boxed{\begin{aligned}
 \dot\tht&=g^{-1}\mom,\\
 \dot\mom&=-\grad L-\Fgeo-\gamma\mom.
 \end{aligned}}
 \label{eq:damped}
\end{equation}
The acceleration form acquires $+\gamma\dot\theta_a$ on the left of
\eqref{eq:geodesic}. When $g=m$ is scalar and constant, $c=\gamma m$.
Writing $\gamma\mom^{\T}\mom/2$ and differentiating with respect to momentum
can reproduce $-\gamma\mom$ algebraically, but that expression is not the
velocity-based Rayleigh construction for a curved metric. In momentum
coordinates, \eqref{eq:rayleigh} equals
$\gamma\mom^{\T} g^{-1}\mom/2$; its velocity derivative is what enters the force.

\subsection{Energy Dissipation relation}
% ------------------------------------------------------------
\begin{theorem}[energy dissipation]\label{thm:energy}
For a smooth, time-independent metric and loss, solutions of
\eqref{eq:damped} satisfy
\begin{equation}
 \dot H_0=-\gamma\mom^{\T} g^{-1}\mom
 =-\gamma\vel^{\T} g\vel=-2\mathcal R\leq0.
 \label{eq:energy-decay}
\end{equation}
\end{theorem}
\begin{proof}
Apply the chain rule and substitute the canonical equations:
\begin{align}
 \dot H_0&=(\partial_{\tht} H_0)^{\T}\dot\tht
 +(\partial_{\mom} H_0)^{\T}\dot\mom\\
 &=(\partial_{\tht} H_0)^{\T}\vel
 +\vel^{\T}(-\partial_{\tht} H_0-\gamma\mom)\\
 &=-\gamma\vel^{\T}\mom.
\end{align}
The Hamiltonian terms cancel. Substitute \eqref{eq:momentum} and use positive
definiteness.
\end{proof}
The loss derivative $\dot L=\grad L^{\T}g^{-1}\mom$ may have either sign.
Thus $\dot H_0\leq0$ does not imply $\dot L\leq0$.
\begin{corollary}[Damped Noether balance]\label{cor:noether-damped}
Solutions of \eqref{eq:damped} satisfy
\begin{equation}
\dot J_\xi=\tfrac12\vel^{\T}(\mathfrak L_\xi g)\vel-\xi^{\T}\grad L-\gamma J_\xi .
\label{eq:noether-damped}
\end{equation}
If $\xi$ is a Killing field of $g$ and a symmetry of the loss,
$J_\xi(t)=e^{-\gamma t}J_\xi(0)$.
\end{corollary}
\begin{proof}
Repeat the proof of Theorem~\ref{thm:noether} with the forced equation
$\dot{\mom}=\partial_{\tht}\mathcal L-\gamma\mom$; the extra term is
$-\gamma\mom^{\T}\xi=-\gamma J_\xi$.
\end{proof}
Dissipation does not break a symmetry. It drains the symmetry's charge at
the same rate $\gamma$ at which it drains momentum, whereas the energy, the
time-translation charge, decreases at the state-dependent rate
$2\mathcal R$.

\subsection{Conditional Stationary-Set Convergence}
% ------------------------------------------------------------
\begin{theorem}[Approach to the stationary set]\label{thm:stationary}
Under Assumption~\ref{ass:core}, let $K\subset\Omega\times\R^n$ be compact
and positively invariant for \eqref{eq:damped}, with $g\succeq mI$,
$m>0$, on its parameter projection. Every solution starting in $K$ obeys
\[
 \operatorname{dist}\big((\tht(t),\mom(t)),\mathcal Z_K\big)\longrightarrow0,
\]
\[
 \mathcal Z_K=\{(\tht,0)\in K:\grad L(\tht)=0\}.
\]
\end{theorem}
\begin{proof}
Compact invariance ensures forward existence. The energy theorem and
LaSalle's invariance principle~\cite{lasalle,khalil-invariance} imply approach to
the largest invariant subset of $K\cap\{\dot H_0=0\}$. Positive definiteness
makes this zero-derivative set $K\cap\{\mom=0\}$. There,
$\dot\tht=0$ and $\dot\mom=-\grad L$, so precisely the points of
$\mathcal Z_K$ can remain in it.
\end{proof}
This is a set-convergence statement.
\begin{remark}[Compactness is a hypothesis]\label{rem:compactness}
Compactness of $K$ is assumed, not derived, and it can fail when the loss has
noncompact symmetries. For the MPS parameterization of Section~\ref{sec:mps},
the gauge transformation $A^{(k)}\mapsto\lambda A^{(k)}$,
$A^{(k+1)}\mapsto\lambda^{-1}A^{(k+1)}$ leaves the image, and hence $L$,
unchanged for every $\lambda>0$, so sublevel sets of $L$ are unbounded along
gauge orbits as $\lambda\to\infty$. A compact positively invariant $K$ is then
not guaranteed, and Theorem~\ref{thm:stationary} is not invoked for the MPS
experiments; their convergence is reported numerically only. Gauge fixing, for
example a canonical MPS form, or a gauge-breaking penalty would be needed to
restore compactness and is not pursued here.
\end{remark}
% ------------------------------------------------------------
\subsection{The Overdamped Limit}
% ------------------------------------------------------------
Put $\varepsilon=\gamma^{-1}$, use slow time $s=\varepsilon t$, and write
$\mom=\varepsilon\vect q$. Because $\Fgeo$ is quadratic in momentum,
\begin{equation}
 \begin{split}
 \frac{\dd\tht}{\dd s}&=g^{-1}\vect q,\\
 \varepsilon^2\frac{\dd\vect q}{\dd s}
 &=-\grad L-\vect q-\varepsilon^2\Fgeo(\tht,\vect q).
 \end{split}
\end{equation}
For bounded smooth coefficients and after the fast initial relaxation,
the leading balance gives $\vect q=-\grad L$. Hence
\begin{equation}
 \frac{\dd\tht}{\dd s}=-g^{-1}\grad L,
 \qquad \dot\tht\simeq-\gamma^{-1}g^{-1}\grad L.
 \label{eq:overdamped}
\end{equation}
The slow-time equation is metric gradient flow; $g=I$ recovers ordinary
gradient flow. This is a formal singular-limit calculation.
\subsection{Adaptive Metric Construction}
% ============================================================
\subsubsection{Hessian Regularization}
% ------------------------------------------------------------
For the quadratic loss $L=\theta_1^2/2+2\theta_2^2$,
\begin{equation}
 B_L=\begin{bmatrix}1&0\\0&4\end{bmatrix},\qquad
 B_L^{-1}\grad L=\begin{bmatrix}\theta_1\\\theta_2\end{bmatrix}.
\end{equation}
The inverse Hessian balances unequal curvature. For a nonconvex loss, however,
a positive scalar shift chosen without reference to $\lambda_{\min}$ does not
make the Hessian positive definite: positive definiteness of $C+\sigma I$
requires $\sigma>-\lambda_{\min}(C)$, so the shift must be curvature dependent.
Given a numerical curvature estimate $B$, define
\begin{align}
 C&=\tfrac12(B+B^{\T}),\\
 \lambda_{\mathrm{eff}}
 &=\delta+\max\{0,-\lambda_{\min}(C)\},\qquad\delta>0,\\
 g_0&=C+\lambda_{\mathrm{eff}}I.
 \label{eq:regularization}
\end{align}
Every eigenvalue of $g_0$ is at least $\delta$. Symmetrization also has an
algebraic justification: if $A=-A^{\T}$, then
$\vect x^{\T} A\vect x=-\vect x^{\T} A\vect x=0$.
The adaptive shift in \eqref{eq:regularization} is only piecewise smooth. At
an eigenvalue crossing or at the switching point of the maximum, ordinary
metric derivatives may not exist. The smooth derivation applies on smooth
branches. A fixed shift valid over the region of interest, or a smooth
positive-definite construction, avoids this issue. When the shift varies
smoothly with $\tht$, its derivative must be included in $\partial_i g$.
\begin{proposition}[Positive-definite curvature shift]
For symmetric $C$ and $\delta>0$, the metric in \eqref{eq:regularization}
satisfies $g_0\succeq\delta I$.
\end{proposition}
\begin{proof}
Its least eigenvalue is
$\lambda_{\min}(C)+\delta+\max\{0,-\lambda_{\min}(C)\}\geq\delta$.
\end{proof}
\subsubsection{Metric Resolution Hierarchy}
\label{sec:metric-hierarchy}
% ------------------------------------------------------------
Metric construction is independent of the variational derivation. A
regularized Fisher metric uses $g_0=F+\delta I$ when $F\succeq0$;
Hessian metrics require the shift or another SPD construction above.
Structured choices trade curvature resolution for storage and solve cost:
\begin{equation}
 g_{\mathrm{full}},\quad g_{\mathrm{block}},\quad g_{\mathrm{Kron}},
 \quad g_{\mathrm{diag}},\quad g_{\mathrm{low\text{-}rank}}.
\end{equation}
Block and Kronecker factors must give an SPD operator. A low-rank-plus-diagonal
choice $g=D+VV^{\T}$ is SPD for positive diagonal $D$.
Every representation retains the relation $g\vel=\mom$, but differentiating
a position-dependent representation is part of its geometric force.

\section{Discrete Dynamics and Stability} \label{Discrete Dynamics and Stability and Phase analysis}
% ============================================================
% ------------------------------------------------------------
\subsection{Canonical Momentum-First Discretization}
% ------------------------------------------------------------
Let $h>0$ and evaluate the force at $(\tht_k,\mom_k)$.
Define
\begin{equation}
 \vect f_k=\vect\ell_k+\Fgeo_k,\qquad\vect\ell_k=\nabla L(\tht_k).
\end{equation}
Forward Euler applied to the momentum equation gives
\begin{align}
 \frac{\mom_{k+1}-\mom_k}{h}&=-\vect f_k-\gamma\mom_k,\\
 \mom_{k+1}&=(1-\gamma h)\mom_k-h\vect f_k.
\end{align}
With $\beta=1-\gamma h$, update position using the new momentum:
\begin{equation}
 \boxed{\begin{aligned}
 \mom_{k+1}&=\beta\mom_k-h\vect f_k,\\
 \tht_{k+1}&=\tht_k+h\,g_k(\tht_k)^{-1}\mom_{k+1}.
 \end{aligned}}
 \label{eq:master}
\end{equation}
The usual range $0\leq\beta<1$ corresponds to $0<\gamma h\leq1$ in this
Euler interpretation. Solving only $\dot\mom=-\gamma\mom$ exactly instead
gives retention $e^{-\gamma h}$. Composing exact damping with a force step
has the same first-order limit but a different splitting convention.
For a frozen constant force, exact integration of the combined momentum
equation would replace $h\vect f_k$ by
$(1-e^{-\gamma h})\vect f_k/\gamma$.
Using a single symbol $\eta$ for both appearances of $h$ in
\eqref{eq:master} is a common convention that can obscure this scaling: if
$\beta=0$, the force is just the gradient, and the metric
is locally frozen, then
\begin{equation}
 \tht_{k+1}=\tht_k-h^2g_k^{-1}\grad L(\tht_k).
 \label{eq:h-squared}
\end{equation}
Its effective gradient-step size is $h^2$, not $h$.
If separate momentum and position step factors are used, their product
replaces $h^2$. Holding $\beta<1$ fixed while $h\to0$ also makes the
Euler-equivalent damping rate $(1-\beta)/h$ diverge; this is not a
fixed-damping continuous-time limit.

\subsection{Quadratic Stability}

To understand the finite-step behavior of HAMLET, consider the quadratic loss
\[
L(\tht)=\frac{1}{2}\tht^{\T}A\tht,
\qquad A\succ0,
\]
with a constant positive-definite metric \(g\succ0\). With no additional
forces, the discrete update reduces to
\[
\tht_{k+1}
=
\big[(1+\beta)I-h^2g^{-1}A\big]\tht_k
-
\beta\tht_{k-1}.
\]

Because \(g^{-1}A\) is similar to the symmetric positive-definite matrix
\(g^{-1/2}Ag^{-1/2}\), the dynamics can be analyzed independently along its
eigendirections. If \(a>0\) denotes one such eigenvalue, the corresponding
scalar mode satisfies
\[
x_{k+1}
=
(1+\beta-h^2a)x_k-\beta x_{k-1}.
\]
The characteristic equation is therefore
\[
z^2-(1+\beta-h^2a)z+\beta=0.
\]

For \(0\leq\beta<1\), both roots lie inside the unit circle precisely when
\[
0<h^2a<2(1+\beta).
\]
Hence the full quadratic system is asymptotically stable when
\[
\boxed{
0<h^2\lambda_i(g^{-1}A)<2(1+\beta)
\qquad\text{for all }i.
}
\]

This expression makes the role of the metric explicit. The relevant
curvature is not the spectrum of \(A\) alone, but that of the transformed
operator \(g^{-1}A\). In particular, if the metric is chosen as the exact
curvature,
\[
g=A,
\]
then
\[
g^{-1}A=I,
\]
and every mode obeys the same recurrence,
\[
x_{k+1}
=
(1+\beta-h^2)x_k-\beta x_{k-1}.
\]
The original curvature imbalance is therefore removed from the modal
dynamics.

For a regularized metric \(g=A+\delta I\), the transformed eigenvalues become
\[
\frac{\lambda_i(A)}{\lambda_i(A)+\delta},
\]
so the regularization controls how completely the curvature is normalized.

The same recurrence also gives the convergence rate. Let
\[
\mu=\lambda_{\min}(g^{-1}A),
\qquad
\Lambda=\lambda_{\max}(g^{-1}A),
\qquad
\kappa=\frac{\Lambda}{\mu}.
\]
Using the classical heavy-ball choice~\cite{polyak1964}
\[
h^2=\frac{4}{(\sqrt{\Lambda}+\sqrt{\mu})^2},
\qquad
\beta=
\left(
\frac{\sqrt{\Lambda}-\sqrt{\mu}}
{\sqrt{\Lambda}+\sqrt{\mu}}
\right)^2,
\]
every mode has roots of modulus $\sqrt\beta$, so the spectral radius of the
recurrence, i.e.\ its asymptotic per-step convergence rate, is
\[
\sqrt{\beta}
=
\frac{\sqrt{\kappa}-1}
{\sqrt{\kappa}+1}.
\]
This is a spectral radius, not a one-step contraction factor in norm: the
companion matrix of the recurrence is not normal, and at the extreme modes
$a=\mu,\Lambda$ its roots coincide, so $|x_k|$ decays as
$O(k\,\beta^{k/2})$ and may grow transiently before contracting.

Thus, for this quadratic model, the convergence rate depends on the condition
number of the transformed operator \(g^{-1}A\). Exact curvature gives
\(\kappa=1\), so all modes contract at the same rate. This is the basic
mechanism behind the curvature normalization observed later in the image
reconstruction example of Section~\ref{Controlled Image Reconstruction Study}.
These results apply to the constant-metric quadratic problem. When the metric
depends on position or changes during training, the update is no longer a
fixed linear recurrence, so the same stability condition does not apply
directly.

\section{MNIST Classification Benchmark}
\label{sec:mnist}
The image-reconstruction study in the following sections is deliberately
controlled. Before turning to it, we check that the HAMLET update also
behaves well on a standard learning task by training a small classifier on
MNIST handwritten digits~\cite{mnist}.

Each optimizer trains the same two-layer perceptron,
$784\rightarrow128\rightarrow10$, with a ReLU hidden layer and a softmax
cross-entropy loss. We use a fixed split of 55{,}000 training, 5{,}000
validation, and the official 10{,}000 test images, and train for 10 epochs
with batch size 128, or 4{,}300 updates. The comparison is paired: within
each of five seeds, every optimizer starts from the same initialization and
sees the same minibatch order. Hyperparameters were tuned on a separate seed
using the final validation negative log-likelihood (NLL), and the test set
played no part in that choice.

We compare a single HAMLET formulation, the core dynamics derived above,
against two standard baselines: SGD with momentum, and Adam with decoupled
weight decay (AdamW)~\cite{adamw}.For HAMLET we use a
diagonal metric built from a running average of squared gradients,
\begin{equation}
\begin{aligned}
\vect\nu_{k+1}&=\rho\,\vect\nu_k+(1-\rho)\,(\grad L_k\odot\grad L_k),\\
g_k&=\diag\!\big(\sqrt{\vect\nu_{k+1}}+\epsilon_g\big).
\end{aligned}
\label{eq:mnist-metric}
\end{equation}
Here $k$ is the training step, $\grad L_k$ the current minibatch gradient,
and $\vect\nu_k$ a running average of past squared gradients, which records
the recent gradient magnitude of every parameter. The decay coefficient
$\rho=0.99$ sets how much of the old average is retained, the constant
$\epsilon_g=10^{-3}$ ensures numerical stability, and $\diag$ places the
elementwise square roots on the diagonal of the metric $g_k$. To keep the method scalable, this
metric is frozen during each optimization step, so $\partial_{\tht}g_k=0$ and
the geometric force vanishes. The benchmark therefore tests the
metric--momentum structure of HAMLET rather than its full position-dependent
dynamics, which are studied on the MPS manifold in Section~\ref{sec:mps}.

Table~\ref{tab:mnist-main} and Fig.~\ref{fig:mnist-comparison} report the
results. HAMLET gives the highest mean test accuracy, $97.95\pm0.09\%$, and
the lowest mean test NLL, $0.0696\pm0.0028$. SGD with momentum reaches
$97.87\pm0.21\%$ and $0.0748\pm0.0070$, while AdamW reaches
$97.28\pm0.24\%$ and $0.1252\pm0.0169$. HAMLET also has the smallest spread
across seeds on every metric. Its margin over SGD with momentum is small,
however: on both test metrics the difference lies within one standard
deviation, and SGD with momentum attains a slightly lower validation NLL. Hence we conclude under
the same fixed tuning budget, the HAMLET optimizertrains a standard neural
classifier reliably and reaches competitive or better predictive metrics.

\begin{table*}[!t]
\caption{MNIST MLP results over five paired evaluation seeds. Entries are
mean $\pm$ sample standard deviation; accuracy is in percent. Lower NLL is
better; higher accuracy is better. Bold marks the best mean in each column.}
\label{tab:mnist-main}
\centering
\begin{tabular}{@{}lccccc@{}}
\toprule
Method & $n$ & Test accuracy (\%) & Test NLL & Best val.\ NLL & Final val.\ NLL\\
\midrule
SGD+momentum & 5 & $97.87\pm0.21$ & $0.0748\pm0.0070$ & $\mathbf{0.0670\pm0.0045}$ & $\mathbf{0.0742\pm0.0104}$\\
AdamW        & 5 & $97.28\pm0.24$ & $0.1252\pm0.0169$ & $0.0851\pm0.0025$ & $0.1241\pm0.0133$\\
HAMLET       & 5 & $\mathbf{97.95\pm0.09}$ & $\mathbf{0.0696\pm0.0028}$ & $0.0677\pm0.0010$ & $0.0743\pm0.0029$\\
\bottomrule
\end{tabular}
\end{table*}

\begin{figure*}[!t]
\centering
\includegraphics[width=0.92\textwidth]{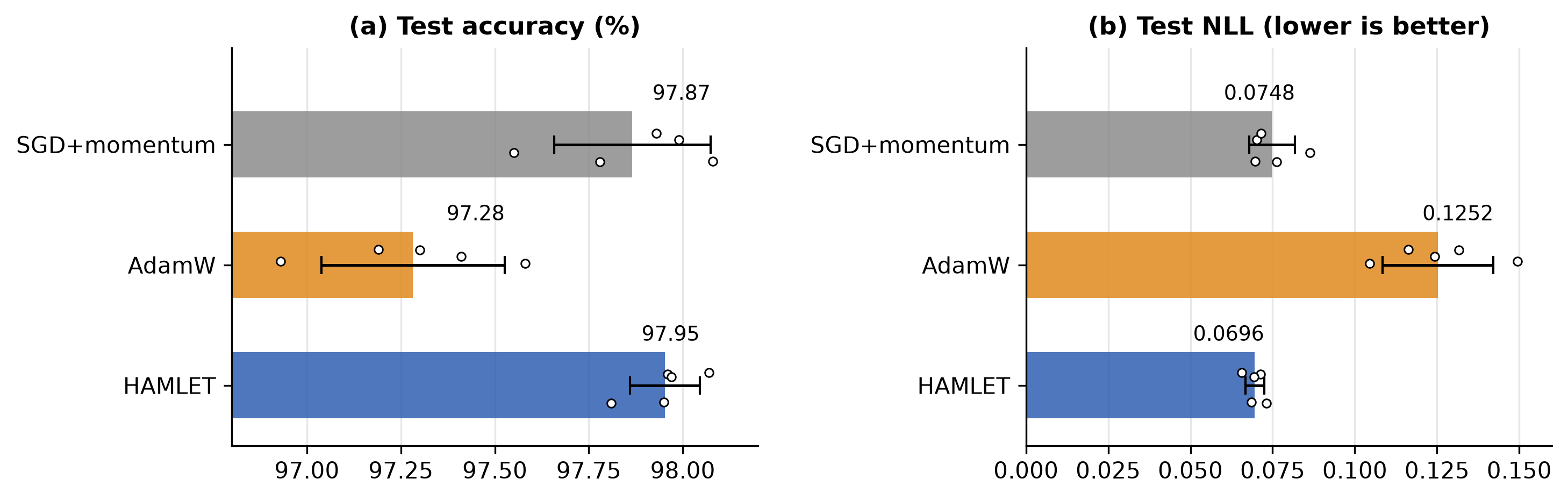}
\caption{MNIST MLP comparison over five paired seeds using one HAMLET
formulation and two standard baselines. Bars show the mean final test metrics,
error bars the sample standard deviation, and open circles the individual
seeds; the numerical values are reported in Table~\ref{tab:mnist-main}.}
\label{fig:mnist-comparison}
\end{figure*}

\section{Controlled Image Reconstruction Study}\label{Controlled Image Reconstruction Study}
\subsection{Image Reconstruction Problem}

We consider a single RGB image from the CIFAR-10 dataset~\cite{cifar10} as a controlled
image-reconstruction problem. The image has spatial resolution
$32\times32$ with three color channels, giving
$32\times32\times3=3072$ real-valued optimization variables after
normalization of the pixel intensities.

Let $x^\star$ denote the original image. A degraded observation is constructed as
\begin{equation}
b = Kx^\star + \epsilon,
\qquad
\epsilon_j \sim \mathcal N(0,0.01^2),
\end{equation}
where $K$ is a separable blur operator applied independently to each color
channel and $\epsilon$ represents additive Gaussian noise. The blur kernel is
chosen as
\[
\frac{1}{16}(1,4,6,4,1)
\]
along each spatial direction.

The reconstructed image $x$ is obtained by minimizing
\begin{equation}
\boxed{
L(x)
=
\frac{1}{2}\|Kx-b\|_2^2
+
\frac{\lambda}{2}\|Dx\|_2^2
+
\frac{\delta}{2}\|x\|_2^2
}
\label{eq:image-loss}
\end{equation}
with $\lambda=0.002$ and $\delta=0.001$. Here, $D$ denotes the discrete
spatial-difference operator. The first term enforces consistency with the
observed image, the second penalizes rapid spatial variation, and the third
provides a small quadratic regularization.
The optimization is carried out directly in the continuous pixel variables.

\subsection{Exact Solution of the Reconstruction Problem}

Differentiating the objective in \eqref{eq:image-loss} gives
\begin{align}
\nabla L(x)
&=
K^{\T}(Kx-b)
+
\lambda D^{\T}Dx
+
\delta x,\\
&=
Hx-c,
\end{align}
where
\begin{equation}
H
=
K^{\T}K
+
\lambda D^{\T}D
+
\delta I,
\qquad
c=K^{\T}b.
\end{equation}

The stationary point therefore satisfies
\[
Hx_{\mathrm{opt}}=c,
\]
and hence
\begin{equation}
\boxed{
x_{\mathrm{opt}}=H^{-1}c
}
\label{eq:image-optimum}
\end{equation}

Since $\delta>0$, the Hessian is positive definite. Indeed, for any
nonzero vector $u$,
\[
u^{\T}Hu
=
\|Ku\|_2^2
+
\lambda\|Du\|_2^2
+
\delta\|u\|_2^2
>0.
\]
The reconstruction problem therefore has a unique minimizer.

Using the stationarity condition $Hx_{\mathrm{opt}}=c$, the objective can
also be written relative to its minimum as
\begin{equation}
L(x)-L(x_{\mathrm{opt}})
=
\frac{1}{2}
(x-x_{\mathrm{opt}})^{\T}
H
(x-x_{\mathrm{opt}}).
\label{eq:image-gap}
\end{equation}
This expression gives a direct measure of the distance from the optimum in
the curvature induced by $H$.

Because the blur and difference operators are periodic, they are diagonal in
the discrete Fourier basis. Writing the spatial frequencies as
$f_u=u/32$ and $f_v=v/32$, the corresponding symbols are
\begin{align}
\widehat K_{uv}
&=
\cos^4(\pi f_u)\cos^4(\pi f_v),\\
\widehat H_{uv}
&=
\widehat K_{uv}^{\,2}
+
4\lambda
\left[
\sin^2(\pi f_u)
+
\sin^2(\pi f_v)
\right]
+
\delta.
\end{align}
The solution is therefore obtained mode by mode as
\begin{equation}
\widehat x_{\mathrm{opt},uv}
=
\frac{
\widehat K_{uv}\widehat b_{uv}
}{
\widehat H_{uv}
}.
\end{equation}

Each color channel is diagonalized independently in the Fourier basis. This
diagonal representation also makes the spectral structure of the image
Hessian explicit, which will be useful when comparing Euclidean and
curvature-matched dynamics.

\subsection{Numerical Reconstruction}

We initialize the reconstruction from
\[
x_0=0,\qquad p_0=0,
\]
and use the same step size and momentum retention for both Euclidean
momentum, \(g=I\), and HAMLET with the curvature metric, \(g=H\).
Since both metrics are constant for this quadratic problem,
\[
\Fgeo=\vect 0.
\]

This comparison isolates the effect of the metric. In the Euclidean case,
different curvature directions evolve at different rates, whereas the choice
\(g=H\) normalizes the quadratic geometry.

\begin{figure*}[!t]
\centering
\includegraphics[width=\textwidth]{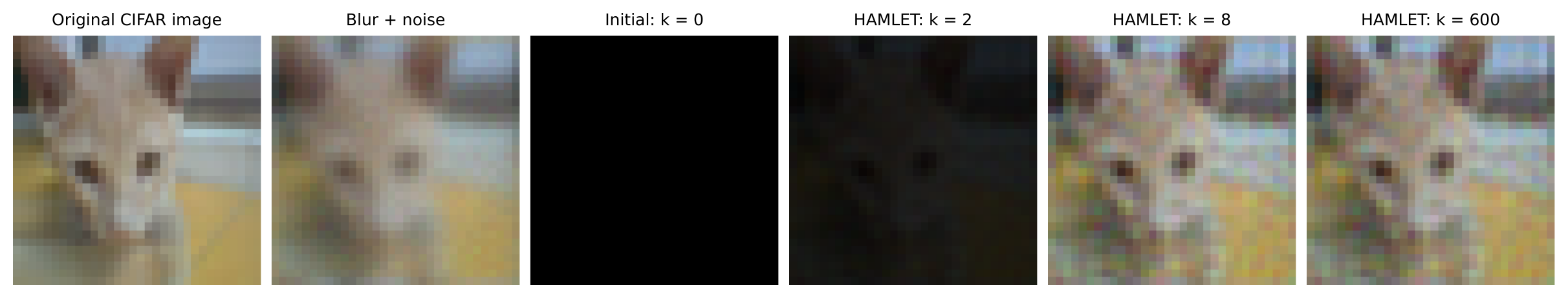}
\caption{Reconstruction of the selected CIFAR-10 image. From left to right:
original image, blurred and noisy observation, zero initialization, two
intermediate HAMLET iterates, and the final reconstruction.}
\label{fig:cifar-sequence}
\end{figure*}

\begin{figure*}[!t]
\centering
\includegraphics[width=\textwidth]{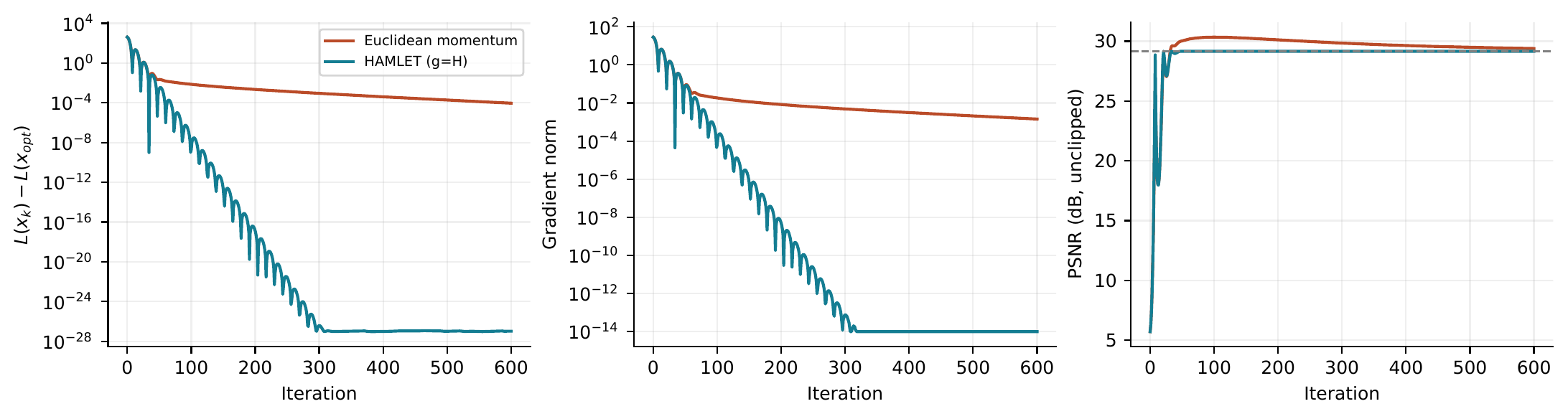}
\caption{Convergence of Euclidean momentum and HAMLET using the same numerical
settings. The panels show the objective gap, gradient norm, and reconstruction
quality during optimization.}
\label{fig:cifar-convergence}
\end{figure*}
HAMLET approaches the known optimum substantially faster in this controlled
quadratic example. The difference is most visible in the objective gap and
gradient norm, which reflect the optimization problem directly. The image
quality measure need not follow the objective monotonically because the
regularized minimum is not defined solely by pixel-wise similarity to the
original image.
\subsection{Energy Evolution}

For the constant-metric problem, define the potential, kinetic, and total
energies as
\[
V_k=L(x_k),
\qquad
T_k=\frac{1}{2}p_k^{\T}g^{-1}p_k,
\qquad
\mathcal H_k=T_k+V_k.
\]

Figure~\ref{fig:cifar-energy} shows the exchange between kinetic and potential
energy during the optimization. The total energy decreases as the trajectory
approaches the stationary state, while kinetic energy becomes small near the
minimum.

This behavior is the discrete counterpart of the continuous dissipation law
\[
\dot{\mathcal H}
=
-\gamma p^{\T}g^{-1}p
\leq0.
\]
The continuous identity is exact, whereas the discrete trajectory depends on
the chosen time step and integration rule.

\begin{figure*}[!t]
\centering
\includegraphics[width=0.96\textwidth]{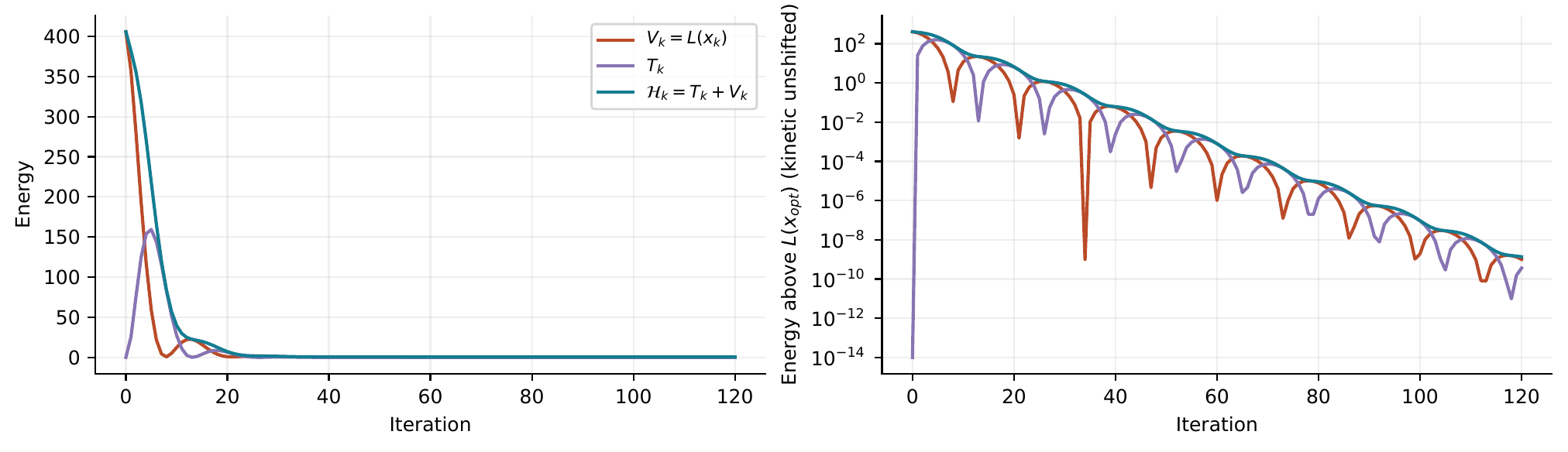}
\caption{Potential, kinetic, and total energy during HAMLET optimization.
The logarithmic representation emphasizes the decay toward the minimum.}
\label{fig:cifar-energy}
\end{figure*}
\section{Spectral Interpretation of Curvature Normalization}
\label{sec:spectral-normalization}

Let
\[
e=x-x_{\mathrm{opt}}.
\]
For the curvature metric \(g=H\), the HAMLET equations become
\[
\dot e=H^{-1}p,
\qquad
\dot p=-He-\gamma p.
\]
Using \(p=H\dot e\), we obtain
\begin{equation}
H\ddot e+\gamma H\dot e+He=0,
\end{equation}
and therefore
\begin{equation}
\boxed{
\ddot e+\gamma\dot e+e=0.
}
\label{eq:image-normalized-ode}
\end{equation}

The significance of this equation is that the Hessian no longer appears in
the modal dynamics. In contrast, Euclidean momentum gives
\[
\ddot e+\gamma\dot e+He=0.
\]
Writing
\[
H=Q\Lambda Q^{\T},
\qquad
e=Qz,
\]
the Euclidean modes satisfy
\[
\ddot z_i+\gamma\dot z_i+\lambda_i(H)z_i=0,
\]
whereas the curvature-matched dynamics reduce to
\[
\ddot z_i+\gamma\dot z_i+z_i=0.
\]

Thus every curvature-normalized mode evolves according to the same scalar
equation, while the Euclidean dynamics retain the spread of Hessian
eigenvalues.
The same normalization appears in discrete time. Defining
\[
q_k=H^{-1}p_k,
\]
gives
\[
q_{k+1}
=
\beta q_k-he_k,
\qquad
e_{k+1}
=
e_k+hq_{k+1}.
\]
Hence
\[
\begin{bmatrix}
e_{k+1}\\
q_{k+1}
\end{bmatrix}
=
\begin{bmatrix}
1-h^2 & h\beta\\
-h & \beta
\end{bmatrix}
\begin{bmatrix}
e_k\\
q_k
\end{bmatrix}.
\]

The same \(2\times2\) map therefore governs every curvature-normalized mode.
For Euclidean momentum, the corresponding map contains
\(h^2\lambda_i(H)\), so different eigendirections evolve at different rates.
\begin{figure}[!t]
\centering
\includegraphics[width=\columnwidth]{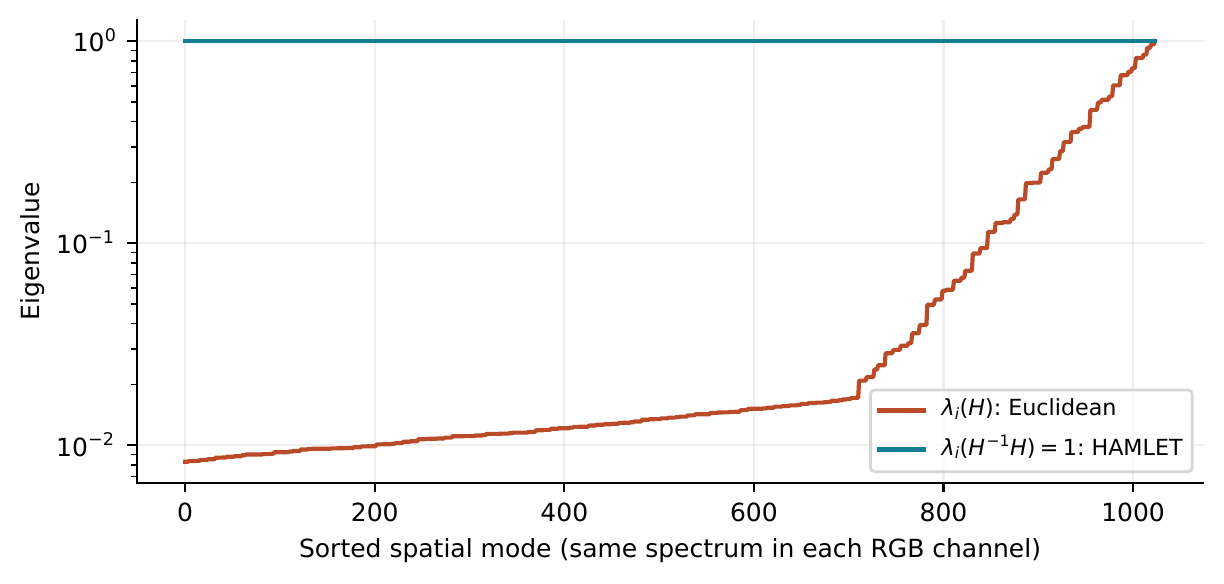}
\caption{Spectrum of the image Hessian before and after curvature
normalization. Choosing \(g=H\) maps the transformed spectrum to unity.}
\label{fig:cifar-spectrum}
\end{figure}

\begin{figure}[!t]
\centering
\includegraphics[width=\columnwidth]{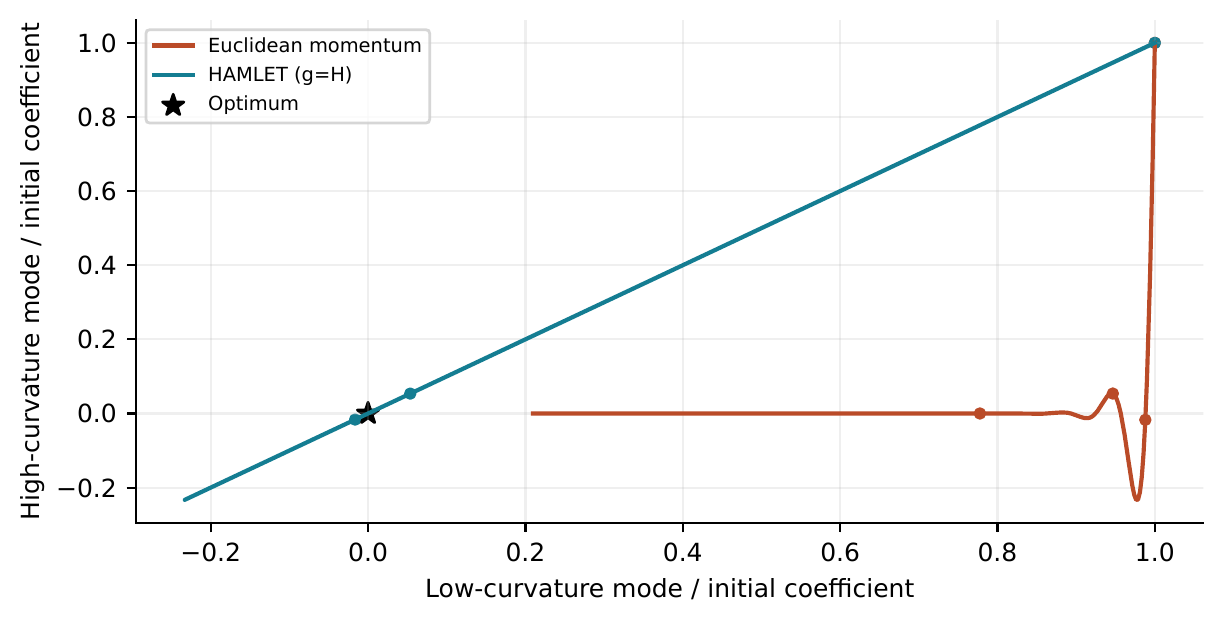}
\caption{Evolution of representative low- and high-curvature modes.
Euclidean momentum relaxes the two modes at different rates, whereas the
curvature-matched HAMLET dynamics collapse them onto the same normalized
trajectory.}
\label{fig:cifar-trajectory}
\end{figure}
This example makes the role of the metric explicit: exact curvature does not
change the location of the minimizer, but it removes the curvature imbalance
from the quadratic dynamics. The resulting modal equations are identical,
which explains the faster contraction observed in the reconstruction
experiment.
\section{HAMLET on a Tensor-Network Parameter Manifold}
\label{sec:mps}

The image reconstruction problem considered previously is quadratic in the
pixel variables and therefore admits a constant Hessian metric. In that
setting the geometric force vanishes. To examine the full geometric dynamics
of HAMLET, we now represent the same image as a matrix product state (MPS).
The physical image remains unchanged, but the optimization variables become
the tensor cores of the MPS. Since the reconstructed image depends
nonlinearly on these cores, the associated metric varies across parameter
space and the geometric-force term becomes active.

\subsection{Matrix Product State Representation}

Let the row and column indices of the $32\times32$ image be written in binary,
\[
r=\sum_{m=1}^{5}r_m2^{5-m},
\qquad
s=\sum_{m=1}^{5}s_m2^{5-m},
\]
with $r_m,s_m\in\{0,1\}$, and let $c\in\{0,1,2\}$ denote the color channel.

The image is represented by an $11$-site matrix product
state~\cite{oseledets-tt,khoromskij-qtt},
\begin{equation}
x_{rsc}
=
A^{(0)}[c]
A^{(1)}[r_1]
A^{(2)}[s_1]
\cdots
A^{(9)}[r_5]
A^{(10)}[s_5].
\label{eq:mps}
\end{equation}

The complete set of tensor cores is denoted by
\[
A=
\left(
A^{(0)},A^{(1)},\ldots,A^{(10)}
\right).
\]
For the numerical example we use a maximum bond dimension
$\chi=8$, resulting in $865$ trainable parameters instead of the
$3072$ pixel variables of the original representation.

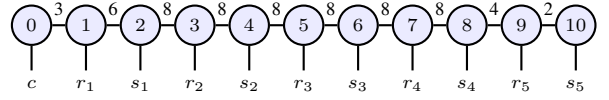
\begin{figure}[!t]\centering
\begin{tikzpicture}[x=0.72cm,y=0.72cm,
 core/.style={circle,draw,thick,fill=blue!8,minimum size=0.52cm,inner sep=0pt,font=\scriptsize},
 leg/.style={font=\scriptsize}]
\foreach \k/\lab in {0/c,1/r_1,2/s_1,3/r_2,4/s_2,5/r_3,6/s_3,7/r_4,8/s_4,9/r_5,10/s_5}{
 \node[core] (A\k) at (\k,0) {$\k$};
 \draw[thick] (A\k.south) -- ++(0,-0.45) node[below,leg] {$\lab$};}
\foreach \k/\b in {0/3,1/6,2/8,3/8,4/8,5/8,6/8,7/8,8/4,9/2}{
 \pgfmathtruncatemacro{\n}{\k+1}
 \draw[thick] (A\k) -- node[above,leg] {\b} (A\n);}
\end{tikzpicture}
\caption{The $32\times32\times3$ cat image as an $11$-site MPS. Circles
are cores $A^{(k)}$, lower legs are the color and binary row/column
indices of \eqref{eq:mps}, and edge labels are bond dimensions for
$\chi=8$.}
\label{fig:mps-diagram}
\end{figure}

\subsection{Image Reconstruction on the MPS Manifold}

Substituting the MPS representation into the image reconstruction objective
defines a loss over the tensor-network parameters,
\[
L(A)\equiv L(x(A)).
\]

Since the exact pixel-space minimizer $x_{\mathrm{opt}}$ is known from the
previous section, the objective can be written as
\begin{equation}
L(x(A))-L(x_{\mathrm{opt}})
=
\frac{1}{2}
\left(x(A)-x_{\mathrm{opt}}\right)^{\T}
H
\left(x(A)-x_{\mathrm{opt}}\right).
\label{eq:mps-reform}
\end{equation}

The MPS optimization problem can therefore be interpreted as searching for
the representation within the chosen tensor-network manifold that best
approximates the pixel-space optimum in the geometry induced by $H$.

Unlike the original pixel problem, however, $x(A)$ is multilinear in the MPS
cores. The optimization landscape in the tensor parameters is therefore
nonlinear and, in general, nonconvex.

\subsection{Local Curvature and the MPS Metric}

Consider a single core $A^{(j)}$ and define
\[
a_j=\operatorname{vec}A^{(j)}.
\]
If all remaining cores are held fixed, the reconstructed image depends
linearly on $a_j$,
\begin{equation}
x=J_j(A)a_j,
\label{eq:mps-jacobian}
\end{equation}
where $J_j$ contains the contraction of the tensor network surrounding
site $j$.

The gradient with respect to the core is therefore
\begin{equation}
\nabla_{a_j}L
=
J_j^{\T}(Hx-c),
\end{equation}
while the corresponding local Hessian block is
\begin{equation}
B_j
=
J_j^{\T}HJ_j.
\label{eq:mps-block-hessian}
\end{equation}

This naturally defines a block-structured HAMLET metric,
\begin{equation}
\boxed{
g(A)
=
\bigoplus_{j=0}^{10}
\left(
J_j^{\T}HJ_j+s_jI
\right)
}
\label{eq:mps-metric}
\end{equation}
where $s_j>0$ is a regularization that keeps each block positive
definite.

The important difference from the pixel-space problem is that $J_j$
depends on the surrounding tensor cores. Consequently,
\[
g=g(A),
\]
and the metric changes as the optimization trajectory moves through the MPS
parameter manifold.

\subsection{Geometric Force on the MPS Manifold}

Because the metric now depends on the tensor-network parameters, the geometric
force is generally nonzero. Writing
\[
v=g^{-1}p,
\]
the geometric contribution to the momentum equation is
\begin{equation}
\boxed{
F_{\mathrm{geo},i}
=
-\frac{1}{2}
v^{\T}
(\partial_i g)
v.
}
\label{eq:mps-fgeo}
\end{equation}

This term measures how the local kinetic geometry changes as the MPS cores
evolve. It is absent in the constant-metric image problem but becomes part of
the dynamics once the optimization is performed on the tensor-network
parameter manifold.

The MPS example therefore provides a direct realization of the general
HAMLET equations,
\begin{equation}
\dot A=g^{-1}(A)p,
\qquad
\dot p
=
-\nabla_A L
-
\Fgeo
-
\gamma p,
\end{equation}
with a position-dependent metric.

\subsection{Connection with Alternating Optimization}

The block metric also provides a simple connection with the alternating
linear scheme commonly used for tensor networks~\cite{holtz-als}. Suppose that only a single
core is updated, the incoming momentum is removed, and the overdamped
unit-step limit is considered. With the block metric
\eqref{eq:mps-metric}, the HAMLET update
$a_j^{+}=a_j-(J_j^{\T}HJ_j+s_jI)^{-1}J_j^{\T}(HJ_ja_j-c)$ then reduces to
\begin{equation}
a_j^{+}
=
\left(
J_j^{\T}HJ_j+s_jI
\right)^{-1}
\left(
J_j^{\T}c+s_ja_j
\right).
\label{eq:mps-als}
\end{equation}

This is the proximally regularized one-site ALS update, i.e.\ the minimizer of
$L(J_ja)+\tfrac{s_j}{2}\|a-a_j\|^2$ over $a$. As $s_j\to0$ it reduces to
the standard one-site ALS update
$a_j^{+}=(J_j^{\T}HJ_j)^{-1}J_j^{\T}c$ whenever $J_j^{\T}HJ_j$ is
invertible; the shift $s_j>0$ keeps the update well defined when this block
is singular or ill conditioned. In this sense, alternating optimization
can be viewed as a sequential overdamped limit of the same local curvature
construction. HAMLET instead retains phase-space dynamics and allows the
cores to evolve with inertia and the geometric-force contribution.

\begin{figure*}[!t]
\centering
\includegraphics[width=0.90\textwidth]{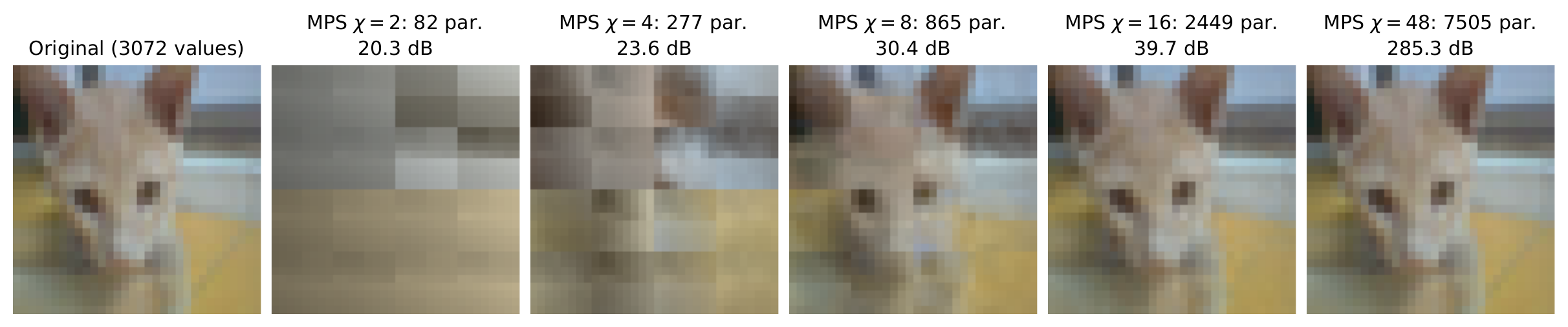}
\caption{Representation of the image at increasing MPS bond dimension.
Larger bond dimension increases the expressive capacity of the tensor
network and allows progressively finer image structure to be retained.}
\label{fig:mps-compression}
\end{figure*}
\subsection{MPS Reconstruction Dynamics}

The MPS reconstruction is initialized from a compressed representation of the
degraded image and evolved under the block metric
\eqref{eq:mps-metric}. Unlike the pixel-space example, the metric now changes
with the optimization variables, so both the curvature and the geometric
force evolve along the trajectory. The reported MPS-HAMLET runs use $h=0.3$,
$\beta=0.8$, and the uniform regularization $s_j=10^3$ for every core $j$.
These values were selected empirically for this finite-step experiment; the
quadratic stability condition of
Section~\ref{Discrete Dynamics and Stability and Phase analysis} is not
asserted globally on the nonlinear MPS trajectory.

Figure~\ref{fig:mps-sequence} shows representative stages of the resulting
reconstruction. The purpose of this example is not to rank tensor-network
optimizers, but to demonstrate the behavior of HAMLET when the parameter
geometry is genuinely position dependent. Because of the MPS gauge freedom,
the convergence guarantee of Theorem~\ref{thm:stationary} does not directly
apply here (Remark~\ref{rem:compactness}); the behavior shown is numerical.

\begin{figure*}[!t]
\centering
\includegraphics[width=0.84\textwidth]{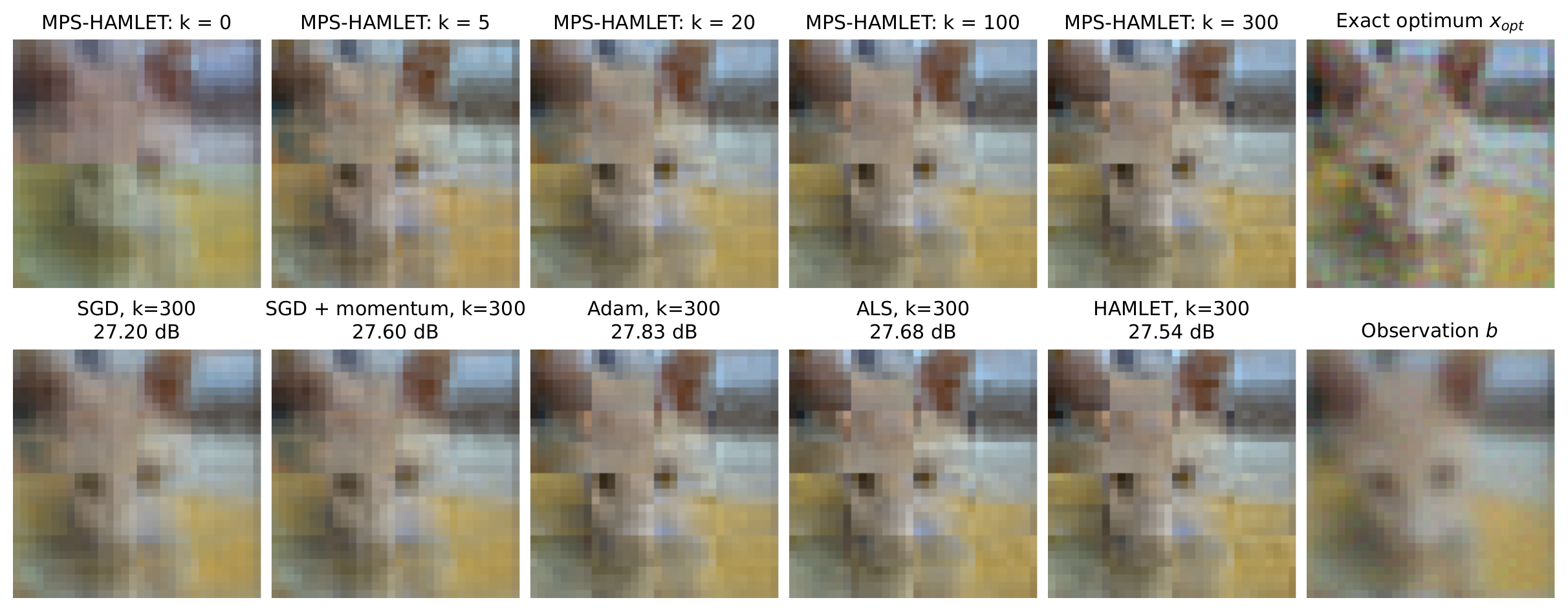}
\caption{MPS reconstruction ($\chi=8$). Top row: MPS-HAMLET iterates at
increasing $k$, evolving toward a low-loss image while the underlying metric
changes with the MPS cores, next to the exact pixel-space optimum
$x_{\mathrm{opt}}$. Bottom row: SGD, SGD with momentum, Adam, ALS, and HAMLET
on the same MPS after $300$ iterations, with PSNR, and the observation $b$.}
\label{fig:mps-sequence}
\end{figure*}
The corresponding energy evolution is shown in
Fig.~\ref{fig:mps-energy}. The potential and kinetic contributions evolve
together as the trajectory approaches a stationary region, while the
geometric-force norm provides a direct measure of the contribution generated
by the changing metric.

\begin{figure}[!t]
\centering
\includegraphics[width=\columnwidth]{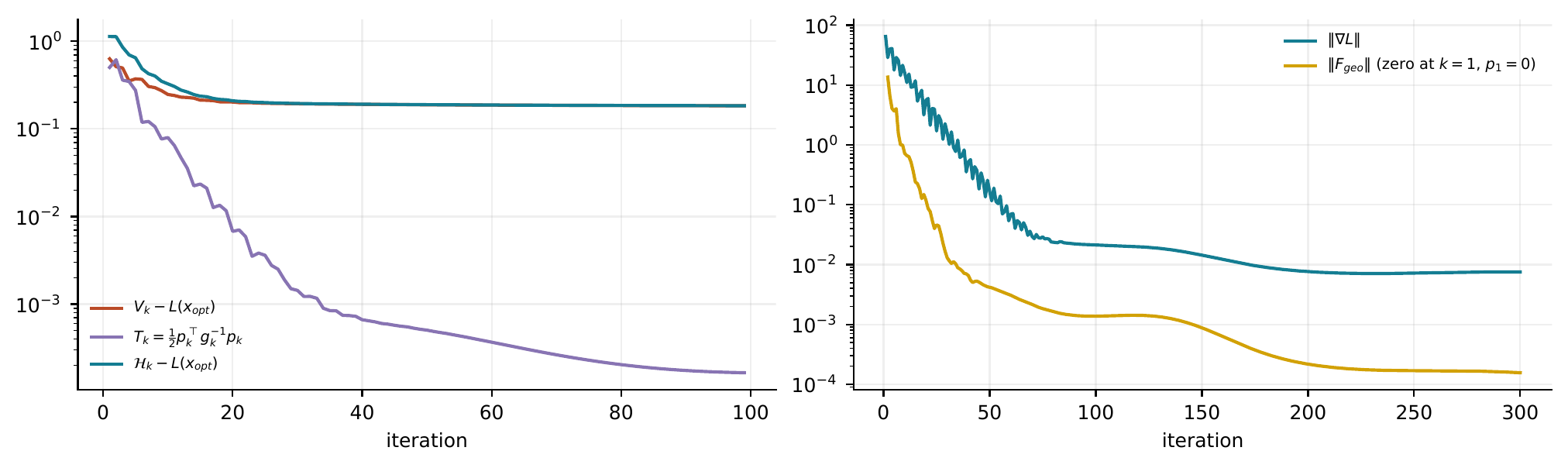}
\caption{Energy and geometric dynamics of HAMLET on the MPS parameter
manifold. The left panel shows the potential, kinetic, and total energy,
while the right panel shows the gradient and geometric-force norms during
the optimization.}
\label{fig:mps-energy}
\end{figure}
\subsection{Interpretation}

The two image representations expose two different regimes of HAMLET. In
pixel space the reconstruction objective has a constant Hessian metric,
\[
g=H,
\]
and therefore
\[
\Fgeo=\vect 0.
\]

In the MPS representation, the same physical image is described through
nonlinear tensor-network coordinates. The resulting metric depends on the
current tensor cores,
\[
g=g(A),
\]
and consequently
\[
\Fgeo\neq\vect 0.
\]

The MPS problem therefore serves as a controlled example of the full
geometric HAMLET dynamics. It shows that the geometric-force term is not an
additional heuristic correction, but follows directly from the
position-dependence of the kinetic metric.

\section{Poincar\'e Analysis of HAMLET Dynamics}
\label{sec:poincare}

The preceding image problems expose two different geometric regimes of
HAMLET. In pixel space the curvature metric is constant, whereas in the MPS
representation the metric depends on the tensor-network parameters. A
Poincar\'e section provides a complementary phase-space view of these
dynamics by recording repeated crossings of a chosen hypersurface rather
than following the complete trajectory.

For a dissipative system the resulting return map is not expected to be
periodic. Instead, it measures how successive crossings contract as the
trajectory approaches a stationary state.

\subsection{Poincar\'e Section and Return Map}

Let $\Sigma$ denote a hypersurface crossed transversally by the HAMLET flow.
A convenient section is defined by the turning points of the loss,
\begin{equation}
\Sigma_L
=
\left\{
(\tht,\mom):
\dot L
=
\nabla L^{\T}g^{-1}\mom
=
0,
\qquad
\ddot L<0
\right\}.
\label{eq:poincare-section}
\end{equation}

The associated first-return map
\[
P:\Sigma_L\rightarrow\Sigma_L
\]
takes one crossing of the trajectory to the next. Because HAMLET is
dissipative,
\[
\dot H_0
=
-\gamma\mom^{\T}g^{-1}\mom
\leq0,
\]
the Hamiltonian is nonincreasing along the flow. Between two crossings of a
nonstationary trajectory $\mom$ cannot vanish identically, so
$H_0(P(z))<H_0(z)$ and $P$ has no periodic points on $\Sigma_L$. Stationary
states do not lie on $\Sigma_L$, since they have $\ddot L=0$ rather than
$\ddot L<0$; they appear only as limit points of successive crossings. Energy decrease alone, however, does not make $P$ a contraction. Throughout this section we
say that $P$ is \emph{contracting} toward a stationary point $z_*$ if there is
$\rho<1$ with $\|P(z)-z_*\|\leq\rho\,\|z-z_*\|$ for all crossings $z$ in a
neighborhood of $z_*$. We establish this exactly only for the linear modal
maps of the quadratic pixel problem in Section~\ref{sec:poincare-modal},
where $\rho=|\mu_i|$. For the nonlinear MPS dynamics we make no general
contraction claim and instead report the empirical ratio of successive
return values.

\subsection{Modal Return Maps in Pixel Space}
\label{sec:poincare-modal}

For the quadratic image problem, each eigendirection evolves independently.
A mode satisfies
\begin{equation}
\ddot e_i
+
\gamma\dot e_i
+
\omega_i^2 e_i
=
0,
\end{equation}
with
\[
\omega_i^2=\lambda_i(H)
\]
for Euclidean momentum and
\[
\omega_i^2=1
\]
for the curvature-matched HAMLET metric.

For an underdamped mode,
\[
\omega_i>\frac{\gamma}{2},
\]
the solution can be written as
\begin{equation}
e_i(t)
=
A_i
e^{-\gamma t/2}
\cos(\Omega_i t+\phi_i),
\qquad
\Omega_i
=
\sqrt{\omega_i^2-\frac{\gamma^2}{4}}.
\end{equation}

Successive turning points are separated by $\pi/\Omega_i$, and the
corresponding return map becomes
\begin{equation}
\boxed{
P_i(e)
=
\mu_i e,
\qquad
\mu_i
=
-\exp\left(
-\frac{\gamma\pi}{2\Omega_i}
\right).
}
\label{eq:continuous-return}
\end{equation}

The magnitude $|\mu_i|$ therefore measures the contraction of a mode between
successive crossings.

For the Euclidean metric, $\Omega_i$ depends on the Hessian eigenvalue
$\lambda_i(H)$, so different curvature directions possess different return
maps. Under the curvature-matched metric, however,
\[
\omega_i=1
\]
for every mode. All modes therefore share the same oscillation frequency and
the same Poincar\'e contraction factor.

\begin{figure*}[!t]
\centering
\includegraphics[width=\textwidth]{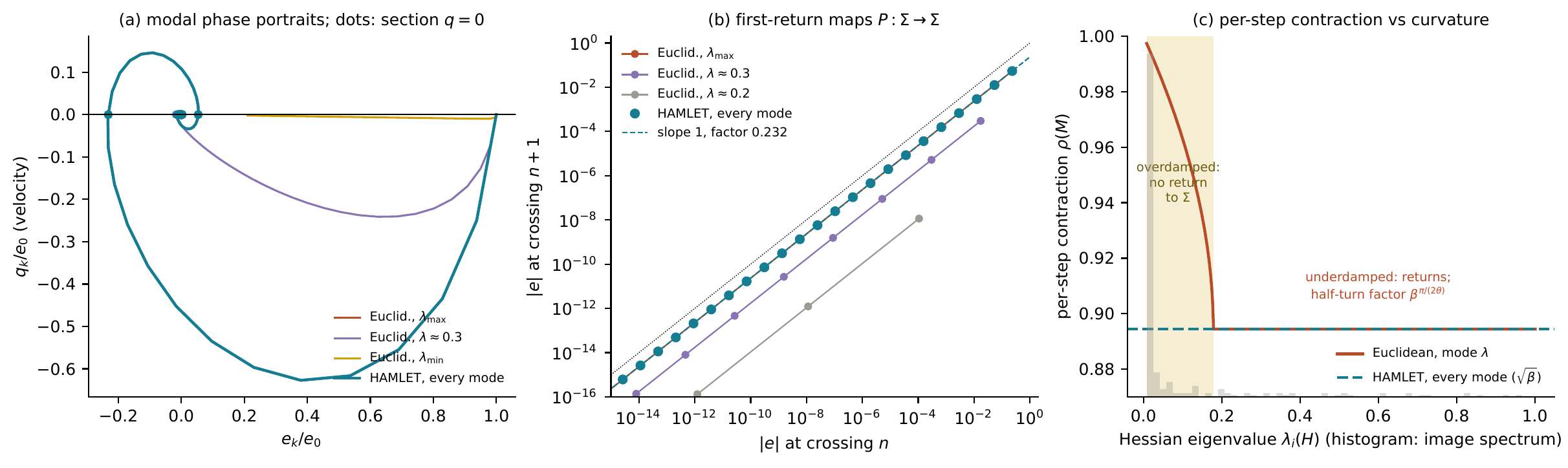}
\caption{Modal Poincar\'e analysis of the pixel-space image problem.
Euclidean momentum produces curvature-dependent phase-space trajectories
and return contractions, whereas the curvature-matched HAMLET metric
collapses all modes onto the same normalized return dynamics.}
\label{fig:modal-poincare}
\end{figure*}
Figure~\ref{fig:modal-poincare} gives a phase-space interpretation of
curvature normalization. With the Euclidean metric, different Hessian
eigenvalues generate different rotations and contraction rates, while
sufficiently low-curvature directions become overdamped and cease to return
to the Poincar\'e section. For the curvature-matched metric $g=H$, all
transformed eigenvalues are equal to one. Consequently, every mode follows
the same normalized phase-space trajectory and possesses the same return
contraction. The Poincar\'e picture therefore provides a dynamical
interpretation of the spectral normalization derived previously.
The same picture appears in discrete time. Defining the modal velocity
$q=g^{-1}p$, the update for a mode with transformed curvature $a$ can be
written as
\begin{equation}
\begin{bmatrix}
e_{k+1}\\
q_{k+1}
\end{bmatrix}
=
\begin{bmatrix}
1-h^2a & h\beta\\
-ha & \beta
\end{bmatrix}
\begin{bmatrix}
e_k\\
q_k
\end{bmatrix}.
\label{eq:modal-map}
\end{equation}

For Euclidean momentum,
\[
a=\lambda_i(H),
\]
whereas curvature matching gives
\[
a=1
\]
for every mode. The determinant of the map is $\beta$, so each modal
phase-space area contracts by the same overall factor, while the metric
determines how this contraction is distributed between rotation and decay.

Low-curvature Euclidean modes can become overdamped and may cease to return
to the Poincar\'e section. Curvature matching removes this imbalance by
placing all modes on the same normalized dynamics.
\subsection{Poincar\'e View of the MPS Dynamics}

The MPS problem does not possess a global set of independent Fourier modes,
because both the loss geometry and the metric vary with the tensor cores.
Nevertheless, the same section \eqref{eq:poincare-section} can be used to
sample the trajectory.

The return points are no longer equally spaced and the empirical ratios of
successive return values are not constant. This reflects the nonlinear and position-dependent nature
of the MPS parameter manifold: as the cores evolve, so do the local
curvature and the geometric force.

The block metric normalizes curvature locally within each tensor core, but
does not remove all couplings between different cores. Consequently, some
directions remain weakly curved or overdamped and stop producing repeated
returns. In this sense, the Poincar\'e map provides a direct geometric view
of the difference between exact curvature normalization in pixel space and
the approximate block curvature used on the MPS manifold.

\begin{figure*}[!t]
\centering
\includegraphics[width=\textwidth]{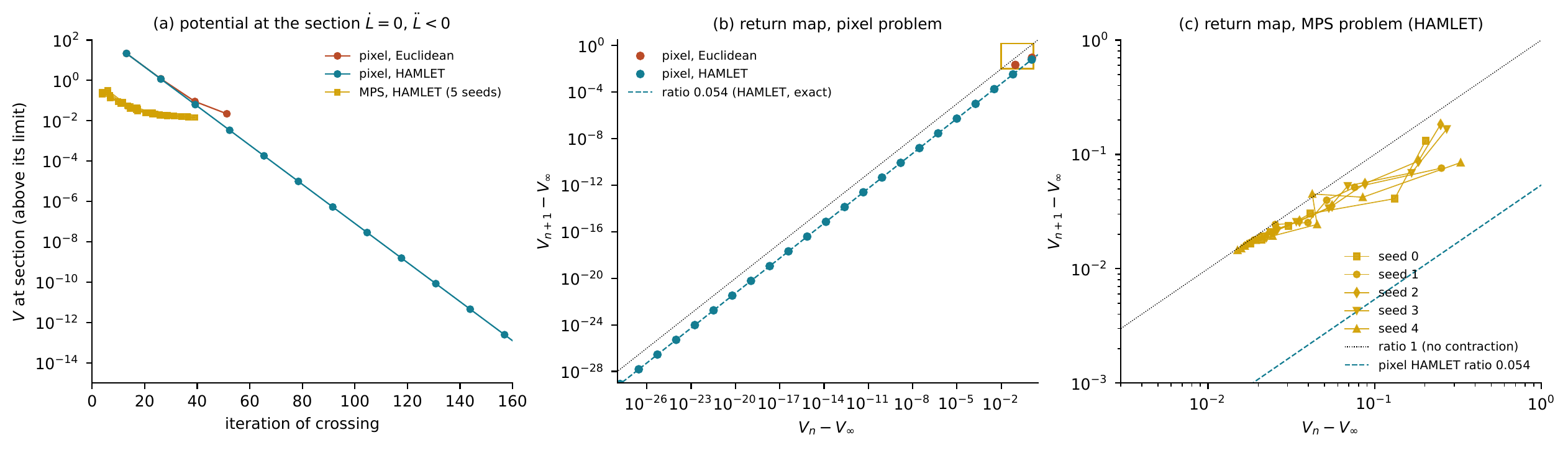}
\caption{Aggregate Poincar\'e return dynamics for the pixel-space and MPS
reconstructions. The curvature-matched pixel problem exhibits a regular
contracting return map, while the MPS trajectory shows irregular early
returns that disappear as overdamped directions dominate the late-time
motion.}
\label{fig:power-poincare}
\end{figure*}
\subsection{Interpretation}

The Poincar\'e analysis gives a phase-space interpretation of the role of
the metric. In the pixel-space problem,
\[
g=H
\]
removes the spread of Hessian eigenvalues, causing every mode to follow the
same damped rotation and the same return map. With the Euclidean metric,
different curvature directions contract at different rates and sufficiently
weak directions become overdamped.

The MPS problem lies between these two limits. Its block metric adapts to the
local tensor-network geometry, but does not fully normalize the complete
coupled Hessian. The resulting return dynamics are therefore nonuniform.

Thus the metric affects not only the direction of parameter motion, but also
the structure of the phase-space contraction itself.

\section{Discussion}
\label{sec:discussion}

The aim of this work is to study optimization as a dissipative dynamical
system evolving on a structured parameter space. Within this formulation,
the loss plays the role of potential energy, while the metric determines
how canonical momentum is converted into parameter velocity. When the
metric depends on the parameters, its variation produces an additional
geometric force that changes the trajectory followed by the optimizer.

The two image-reconstruction examples illustrate complementary regimes of
this construction. In the direct pixel representation, the quadratic
objective admits a constant curvature metric. Choosing the Hessian as the
metric removes the spread of curvature eigenvalues from the modal dynamics,
so that all directions evolve according to the same normalized equation.
This provides a simple setting in which the role of the metric can be seen
analytically.

The MPS representation changes the problem qualitatively. Although the
underlying image is the same, the tensor-network coordinates produce a
position-dependent metric and therefore activate the geometric-force term.
The resulting dynamics provide a controlled example of HAMLET on a genuinely
nonlinear parameter manifold.

The Poincar\'e analysis gives a complementary phase-space interpretation.
For the curvature-matched pixel problem, all modes possess the same return
dynamics. With a Euclidean metric, the return behavior depends on curvature
and weak directions may become overdamped. On the MPS manifold the return
structure becomes nonuniform, reflecting the fact that the block metric
captures local curvature without fully normalizing the coupled parameter
space.

The MNIST benchmark supplies a complementary learning result. HAMLET achieves
the best mean test accuracy and test NLL of the evaluated configurations,
showing that the metric--momentum mechanism remains viable in a standard
classification network. This provides an initial empirical demonstration of
the single optimizer derived in the mathematical sections, while the limited
five-seed, single-architecture scope motivates broader future evaluation.

Taken together, these results suggest that the metric should be viewed not
only as a preconditioning operator, but as part of the dynamical structure
that determines the trajectory, inertia, and phase-space contraction of the
optimizer.
\section{Conclusion}
\label{sec:conclusion}

We introduced HAMLET as a framework in which machine-learning optimization
is formulated as dissipative Hamiltonian motion on a parameter manifold.
Starting from a metric kinetic energy and the loss as potential energy, we
derived the canonical momentum, Hamiltonian equations, geometric force, and
metric-compatible dissipation law.

The framework was evaluated on an MNIST MLP, where HAMLET reached 97.95\% mean
test accuracy and 0.0696 mean test NLL over five paired seeds. It was then
examined on a quadratic image-reconstruction problem,
where an exact curvature metric normalizes the modal dynamics, and then on an
MPS parameterization of the same image, where the metric becomes
position dependent and the geometric force becomes nonzero. Poincar\'e
sections provided an additional phase-space interpretation of how the metric
changes the contraction of different modes.

These results establish the mathematical core of HAMLET and provide a basis
for studying more scalable geometric and dynamical optimization methods.
\section*{Acknowledgements}

Sayan Mukherjee is supported by the Center of Innovation for Sustainable
Quantum AI (JST Grant Number JPMJPF2221).

\bibliographystyle{IEEEtran}
\bibliography{references1}
\end{document}